\documentclass{article}

     \usepackage[preprint]{neurips_2025}

\usepackage[utf8]{inputenc} % allow utf-8 input
\usepackage[T1]{fontenc}    % use 8-bit T1 fonts
\usepackage{hyperref}       % hyperlinks
\usepackage{url}            % simple URL typesetting
\usepackage{booktabs}       % professional-quality tables
\usepackage{amsfonts}       % blackboard math symbols
\usepackage{nicefrac}       % compact symbols for 1/2, etc.
\usepackage{microtype}      % microtypography
\usepackage{xcolor}         % colors

\usepackage{amsthm}
\usepackage{amsmath,amsfonts}
\usepackage{algorithmic}
\usepackage{algorithm}
\usepackage{array}
\usepackage[caption=false,font=normalsize,labelfont=sf,textfont=sf]{subfig}
\usepackage{textcomp}
\usepackage{stfloats}
\usepackage{url}
\usepackage{verbatim}
\usepackage{graphicx}
\graphicspath{{fig/}}

\newtheorem{proposition}{Proposition}
\newtheorem{definition}{Definition}
\newtheorem{remark}{Remark}%

\usepackage{lipsum}

\usepackage{float}
\usepackage{caption}
\graphicspath{{fig/}}

\usepackage{threeparttable}
\usepackage{multirow}
\usepackage{makecell}
\usepackage{natbib}
\setcitestyle{numbers,square}
\title{SETrLUSI: Stochastic Ensemble Multi-Source Transfer Learning Using Statistical Invariant}
\author{
	Chunna Li \\ % 姓名
	Hainan University\\ % 机构
	\texttt{na1013na@163.com} \\ % 邮箱
	\And % 多个作者的分隔符
	Yiwei Song \\ % 姓名
	Hainan University \\ % 机构
	\texttt{syw0917@163.com} \\ % 邮箱
	\And % 多个作者的分隔符
	Yuanhai Shao* \\ % 姓名
	Hainan University \\ % 机构
	\texttt{shaoyuanhai21@163.com} \\ % 邮箱
}
\begin{document}
	\maketitle		
	\begin{abstract}
%		In transfer learning, a source domain often carries diverse knowledge, and different domains usually emphasize different types of knowledge. However, most of the existing methods typically handle only a single type of knowledge from all domains, limiting their effectiveness in tasks. By introducing a weak mode of convergence in the form of Statistical Invariant (SI), this paper proposes a novel ensemble learning framework, formulated as Stochastic Ensemble Multi-Source Transfer Learning Using Statistical Invariant (SETrLUSI). Variant types of knowledge from both source domains and target domain are extracted and integrated by SI, the introduction of which not only exploits different types of knowledge effectively but also expedites the learning process. To accelerate the training speed as well as enhance the model stability, stochastic selection of SI, the proportional sampling of the source domains, and bootstrapping of the target domain are employed in the SETrLUSI. Experiments show that SETrLUSI has good convergence and outperforms related methods with a lower time cost.
	In transfer learning, a source domain often carries diverse knowledge, and different domains usually emphasize different types of knowledge. Different from handling only a single type of knowledge from all domains in traditional transfer learning methods, we introduce an ensemble learning framework with a weak mode of convergence in the form of Statistical Invariant (SI) for multi-source transfer learning, formulated as Stochastic Ensemble Multi-Source Transfer Learning Using Statistical Invariant (SETrLUSI). The proposed SI extracts and integrates various types of knowledge from both source and target domains, which not only effectively utilizes diverse knowledge but also accelerates the convergence process. Further, SETrLUSI incorporates stochastic SI selection, proportional source domain sampling, and target domain bootstrapping, which improves training efficiency while enhancing model stability. Experiments show that SETrLUSI has good convergence and outperforms related methods with a lower time cost.

	\end{abstract}

\section{Introduction}
Transfer learning, as an emerging learning paradigm, aims to leverage knowledge from related domains (source domains) to address the same or similar tasks in another domain (target domain) \cite{chen2023transfer, li2023novel}, thereby reducing the number of labeled samples required for the target task \cite{zhuang2020comprehensive}. \par
%		\begin{figure}[H]
	%			\centering
	%			\includegraphics[height=0.17\textheight]{transfer.pdf}
	%			\caption{A simple illustration of knowledge transfer.}
	%			\label{example}
	%		\end{figure}

%In transfer learning, the primary challenge lies in how to determinate the type of knowledge to transfer, which has been solved by many researchers from different aspects \cite{bruzzone2009domain,wachinger2016domain,jing2020adaptive,yao2022discriminative,zhuang2017supervised,yang2007cross, loey2021hybrid}. 
%\begin{figure}[H]
%	\centering
%	\includegraphics[scale=1.2]{int.pdf}
%	\caption{Two commonly used knowledge transfer types}
%	\label{Int}
%\end{figure}
Multiple source domains usually contain various types of knowledge, even within a single domain \cite{HUANG2023110233, zhao2023novel}. Multi-source domain methods can usually be categorized into two types. The first type methods transfer the same type of knowledge from different source domains, such as Sample type of knowledge \cite{gao2021multi, gao2023integrating, gao2024multi}, Feature type of knowledge \cite{ren2022multi, xu2023multi, li2024subspace, wu2020iterative, zhang2022data}, Parameter type of knowledge \cite{10172347}, and Model type of knowledge \cite{zheng2023multiple, bai2022three}. Those approaches merely increase the number of source domains while maintaining the same knowledge type, without fundamentally expanding the categories of knowledge. To address the issue of knowledge homogeneity, \cite{zhou2021domain, nguyen2021stem,xu2024multi, li2024multi} try to extract different types of knowledge from different source domains. However, they fail to consider intra-domain knowledge diversity within individual source domains and also ignore the more effective utilization of the target domain.

Motivated by these observations, Stochastic Ensemble Multi-Source Domain Transfer Learning Using Statistical Invariant (SETrLUSI) is proposed to fully exploit the various types of knowledge from source domains as well as the target domain by utilizing Statistical Invariant (SI) and the stochastic ensemble. In summary, the main contributions of this paper are as follows:
\begin{itemize}
	\item By embedding knowledge in the way of the weak mode of convergence in a Hilbert space as statistical invariant constraints, it enables extracting diverse forms of knowledge from each heterogeneous source domain simultaneously. Moreover, the utilization of statistical invariants shrinks the admissible function set, thereby facilitating the learning process of the algorithm.
	
	\item  A stochastic transfer ensemble learning framework for multi-source domain problem is further put forward based on the above transfer learning using statistical invariants, which bootstraps samples from source domains for more diverse knowledge to form the weak learners, proportional samples the target domain for more concentration on the target domain, and random selects the statistical invariants for effective knowledge extraction. Experiments on real datasets demonstrate the convergence behavior of our method, as well as the good performance with low computation cost. 
	
\end{itemize}

\section{Stochastic ensemble muli-source domain transfer learning using statistical invariant}
%Based on the aforementioned limitations in current research, we formally propose Multi-Source Stochastic Transfer Learning Using Statistical Invariant (SETrLUSI) and its corresponding ensemble algorithm for addressing these issues.
%\subsection{Preliminary}
%Typically, a domain usually is defined by two principal components: the feature space of inputs $\mathcal{X}$ and the marginal probability distribution of inputs $P(\boldsymbol{X})$ \cite{wang2022contrastive}. Here $\boldsymbol{X}\in\mathcal{X}$ represents a set of learning samples, while $\boldsymbol{Y}$ denotes the corresponding label vector associated with $\boldsymbol{X}$. This study is devoted to address a binary classification transfer problem with multi-source domain $D_{\mathcal{S}^i}, i=1,...,N$ and one target domain $\mathcal{D_T}$. \par	

In this paper, we consider the binary transfer learning classification problem. Let $\mathcal{Q}_i$ and $\mathcal{P}$ represent the marginal probability distributions of the $i$-th source domain and the target domain, respectively. Assume that $\mathcal{Q}_i\ne \mathcal{P}$, but $P(\boldsymbol{Y}_{\mathcal{S}^i}|\boldsymbol{X}_{\mathcal{S}^i})=P(\boldsymbol{Y}_\mathcal{T}|\boldsymbol{X}_\mathcal{T})$, where $\boldsymbol{X}\subset \mathbb{R}^d$ represents a set of learning samples, $\boldsymbol{Y}$ denotes the corresponding label vector associated with $\boldsymbol{X}$, and the subscript $\mathcal{S}^i$ or $\mathcal{T}$ denotes the source domain or the target domain. Denote the datasets come from the $i$-th source domain $\mathcal{S}^i$ as $D_{\mathcal{S}^i}=\{(\boldsymbol{x}_{\mathcal{S}^{i}_k},y_{\mathcal{S}^{i}_k})\}_{k=1}^{N}$, where $\boldsymbol{x}_{\mathcal{S}^{i}_k} \in \boldsymbol{X}_{\mathcal{S}^i}$, $y_{\mathcal{S}^{i}_k} \in \{0,1\}$ being the label for $\boldsymbol{x}_{\mathcal{S}^{i}_k}$.
Similarly, the labeled and unlabeled datasets from the target domain are denoted by $D_{\mathcal{T}^l}=\{(\boldsymbol{x}_{\mathcal{T}^l_k},y_{\mathcal{T}^l_k})\}_{k=1}^{q}$ and $D_{\mathcal{T}^u}=\{\boldsymbol{x}_{\mathcal{T}^u_k}\}_{k=q+1}^{m}$ respectively, where $\boldsymbol{x}_{\mathcal{T}^l_k},\boldsymbol{x}_{\mathcal{T}^u_k} \in \boldsymbol{X}_\mathcal{T}$ and  $y_{\mathcal{T}_{k}^l} \in \{0,1\}$. Our goal is to predict the label of $D_{\mathcal{T}^u}$ by training a learner from $D_{\mathcal{T}^l}$. For convenience, the symbols without subscript $\mathcal{T}$ and $l$ are considered as belonging to the labeled target domain.
%Especially, excepting LUSI and our method, labels of the other methods belong to $\{-1,1\}$. All vectors in the paper are column ones.
%For convenience, some of the basic notations are given in Tabel \ref{notation} in the following context, and the symbols without subscript are considered as belonging to the target domain.
	\subsection{Knowledge transfer based on the weak mode of convergence}
		In classification tasks, compared to hard labels, it is preferable for the learner to provide a probability that an input $\boldsymbol{x}$ belongs to a given class \cite{platt1999probabilistic}. Therefore, our goal is to estimate the probability decision function $f(\boldsymbol{x})=P(y=1|\boldsymbol{x})=\boldsymbol{w}^\top\boldsymbol{x}+b$, which classifies an unlabeled target domain sample $\boldsymbol{x}$ as class 1 if $f(\boldsymbol{x})\ge0.5$ and 0 otherwise. In order to estimate  $f(\boldsymbol{x})$, noting the equality $P(y=1|\boldsymbol{x})P(\boldsymbol{x})=P(y=1,\boldsymbol{x})$, it naturally yields the following Fredholm integral equation	
		\begin{equation}\label{frd}
			\int 			\boldsymbol{\theta}(\boldsymbol{x}-\boldsymbol{\hat{x}})f(\boldsymbol{\hat{x}})\text{d}P(\boldsymbol{\hat{x}})=\int\boldsymbol{\theta}(\boldsymbol{x}-\boldsymbol{\hat{x}})\text{d}P(y=1,\boldsymbol{\hat{x}}),
		\end{equation}		
		where $\boldsymbol{\theta} (\boldsymbol{x}-\boldsymbol{\hat{x}}) =\prod_{k=1}^n{\theta ( x^k-\hat{x}^{k} )}$, 
		$\theta (x^k-\hat{x}^k) =\left\{ \begin{array}{l}
			1,~ \text{if}~x^k\ge 	\hat{x}^k\\
			0,~ \text{if}~x^k <\hat{x}^k\\
		\end{array} \right.$, $x^k$ and $\hat{x}^k$ are the $k$-th 	component of $\boldsymbol{x}$ and $\boldsymbol{\hat{x}}$, respectively 	\cite{vapnik2015v}. Thus, to estimate the conditional probability 
		$f(\boldsymbol{x})$, we aim to minimize a loss $\ell (\cdot)$ between the left and right hands in \eqref{frd}, that is 
		\begin{equation}\label{VSVM}
			\underset{f}{\textup{min}}~\int \ell\left(\int 					\boldsymbol{\theta}(\boldsymbol{x}-\boldsymbol{\hat{x}})f(\boldsymbol{\hat{x}})\text{d}P(\boldsymbol{\hat{x}}),\,\int\boldsymbol{\theta}(\boldsymbol{x}-\boldsymbol{\hat{x}})\text{d}P(y=1,\boldsymbol{\hat{x}})\right)\sigma(\boldsymbol{x})\text{d}\mu(\boldsymbol{x}),
		\end{equation}
		where $\sigma(\boldsymbol{x})$ is a given weight function and $\mu(\boldsymbol{x})$ is a probability measure. Solving problem \eqref{VSVM} to get $f(\boldsymbol{x})$ is considered as a means of achieving the strong mode of convergence. To fully exploit knowledge provided by multiple source domains and the target domain, and also to overcome the limitation of insufficient knowledge utilization in the existing methods, we introduce the weak mode of convergence. 
%		providing a more flexible and theoretically grounded approach to knowledge transfer under data-scarce conditions.
		\begin{definition}{\cite{vapnik2020complete}}
%			{}
		The sequence of functions $\{f_h(\boldsymbol{x})\}_{h=1}^{\infty}$ converges to a function $f(\boldsymbol{x})$ in the weak mode if
		\begin{equation}\label{weakcon}
			\underset{h\rightarrow \infty}{\lim} \langle \psi(\boldsymbol{x}) ,\left\{ f_h(\boldsymbol{x})-f(\boldsymbol{x})\right\}\rangle =0,~\forall~\psi(\boldsymbol{x})\in L_2,
		\end{equation}
		where $L_2$ denotes the space of square-integrable functions.
%			\begin{equation}\label{weakcon}
%				\underset{l\rightarrow\infty}{\textup{lim}} 			\int\psi(\boldsymbol{x})f_r(\boldsymbol{x})\text{d}P(\boldsymbol{x})=\int\psi(\boldsymbol{x})\text{d}P(\boldsymbol{x},y),~\forall\psi(\boldsymbol{x})\in L_2,
%			\end{equation}
%		holds true for all functions $\psi(\boldsymbol{x})\in L_2$, where $P(\pmb{x})$ and $P(\pmb{x}, y)$ are probability distribution and joint distrubution function, $L_2$ is the  square-integrable space \cite{vapnik2020complete}.
		\end{definition}
		In practical scenarios where an infinite number of functions $\psi(\boldsymbol{x})$ is unattainable, it is considered to relax the conditions of the weak mode of convergence by utilizing only a finite number of $\psi(\boldsymbol{x})$ as
		\begin{equation}\label{predicate}
			\langle \psi(\boldsymbol{x}) ,\left\{ f_h(\boldsymbol{x})-f(\boldsymbol{x})\right\}\rangle =0.
		\end{equation}
		In \eqref{predicate}, a $\psi(\boldsymbol{x})$ can be regarded as the representation of some knowledge about the data and is called a predicate \cite{vapnik2020complete}. In this situation, for our concerned problem where $f(\boldsymbol{x}) = P(y = 1|\boldsymbol{x})$, \eqref{predicate} can be rewritten as 
		\begin{equation}
			\int\psi(\boldsymbol{x})P(y=1|\boldsymbol{x})\text{d}P(\boldsymbol{x})=\int\psi(\boldsymbol{x})\text{d}P(y=1,\boldsymbol{x}),
		\end{equation}
		which is called Statistical Invariant (SI) \cite{vapnik2020complete}. Naturally, in transfer learning, it is expected to incorporate knowledge from the source and target domains into the learning process in the form of predicates. Therefore, we consider the following statistical invariants
		\begin{equation}\label{predicateS}
			\int\psi_{\mathcal{S}^i_v}(\boldsymbol{x})f(\boldsymbol{x})\text{d}P(\boldsymbol{x})=\int\psi_{\mathcal{S}^i_v}(\boldsymbol{x})\text{d}P(y=1,\boldsymbol{x}),~i=1,...,N,~v=1,...,V,\\
		\end{equation}
		\begin{equation}\label{predicateT}
			\int\psi_{\mathcal{T}_u}(\boldsymbol{x})f(\boldsymbol{x})\text{d}P(\boldsymbol{x})=\int\psi_{\mathcal{T}_u}(\boldsymbol{x})\text{d}P(y=1,\boldsymbol{x}),~u=1,...,U,
		\end{equation}
		where $\psi_{\mathcal{S}^i_v}(\boldsymbol{x})$  represent the knowledge from the multiple source domains $D_{\mathcal{S}^i}$, while the ones in \eqref{predicateT} extract the knowledge from the target domain. The statistical invariants \eqref{predicateS} and \eqref{predicateT} are the realization of the weak mode of convergence, which restricts the admissible function set of $f(\boldsymbol{x})$ and accelerates convergence \cite{LTKDE}. Therefore, to combine with \eqref{VSVM}, \eqref{predicateS}, and \eqref{predicateT}, we have the following optimization problem for the multi-source transfer learning problem.
	\begin{equation}\label{STLUSI-ori}
		\begin{aligned}
			&\underset{f}{\text{min}} \int \ell\left(\int 		\boldsymbol{\theta}(\boldsymbol{x}-\boldsymbol{\hat{x}})f(\boldsymbol{\hat{x}})\text{d}P(\boldsymbol{\hat{x}}),\,\int\boldsymbol{\theta}(\boldsymbol{x}-\boldsymbol{\hat{x}})\text{d}P( y=1,\boldsymbol{\hat{x}})\right)\sigma(\boldsymbol{x})\text{d}\mu(\boldsymbol{x})\\
			&\text{s.t.}~								\int\psi_{\mathcal{S}^i_v}(\boldsymbol{x})f(\boldsymbol{x})\text{d}P(\boldsymbol{x})=\int\psi_{\mathcal{S}^i_v}(\boldsymbol{x})\text{d}P(y=1,\boldsymbol{x}), ~i=1,...,N,~~v=1,...,V,\\
			&~~~~~\int\psi_{\mathcal{T}_u}(\boldsymbol{x})f(\boldsymbol{x})\text{d}P(\boldsymbol{x})=\int\psi_{\mathcal{T}_u}(\boldsymbol{x})\text{d}P(y=1,\boldsymbol{x}),~~u=1,...,U.
		\end{aligned}
	\end{equation}
	Once the loss function $\ell (\cdot)$ is determined, the optimization problem \eqref{STLUSI-ori} can be solved using existing methods efficiently. The construction of the predicates can be found in Section \ref{Predicate}. \par

\subsection{Knowledge transfer by stochastic ensemble}
	Building on \eqref{STLUSI-ori}, we are able to bring a substantial amount of knowledge from source domains and the target domain. However, it is costly to leverage all knowledge at once, and it should also be noted that not all knowledge is beneficial to the learner \cite{zhang2022survey}. For those reasons, we propose Stochastic Ensemble Multi-Source Transfer Learning Using Statistical Invariant (SETrLUSI), which is an ensemble learning method by utilizing both strong and weak modes of convergence, and stochastic selection. We first construct the following optimization problem for a weak learner based on \eqref{STLUSI-ori}.
	\begin{equation}\label{SETrLUSI-one}
		\begin{aligned}
			&\underset{f}{\text{min}} \int \ell\left(\int 	\boldsymbol{\theta}(\boldsymbol{x}-\boldsymbol{\hat{x}})f(\boldsymbol{\hat{x}})\text{d}P(\boldsymbol{\hat{x}}),\,\int\boldsymbol{\theta}(\boldsymbol{x}-\boldsymbol{\hat{x}})\text{d}P(y=1,\boldsymbol{\hat{x}})\right)\sigma(\boldsymbol{x})\text{d}\mu(\boldsymbol{x})\\
			&\text{s.t.}~		\int\psi^*(\boldsymbol{x})f(\boldsymbol{x})\text{d}P(\boldsymbol{x})=\int\psi^*(\boldsymbol{x})\text{d}P(y=1,\boldsymbol{x}),\\
		\end{aligned}
	\end{equation}
	where $\psi^*(\boldsymbol{x})$ is the predicate with one type of knowledge from source domains or the target domain. We give three stochastic ways to facilitate the utilization of knowledge.

\begin{itemize}
	
	\item \textbf{Stochastic target domain}.
	In classical transfer learning models, limited $D_{\mathcal{T}^l}$ is directly used for training, but it often fails to provide sufficient target domain information. Therefore, we use the bootstrapping \cite{zhao2024model} to stochastic sampling from $D_{\mathcal{T}^l}$ to form $D_{\mathcal{T}^l}^{new}$ in each iteration, which decreases the influence on weak learners from the lack of label samples and also greatly reduces the time cost. \par 	
	
	\item \textbf{Stochastic source domain}.
	According to the statistical heterogeneity of the population, different parts of the same population with the  knowledge that always have different statistical properties. For each SI, we utilize the proportional sampling by ratio $\gamma$ on each source domain $D_{\mathcal{S}^i}$ for contributing to greater diversity in the construction of weak learners. \par
	
	\item \textbf{Stochastic statistical invariant}.
	To ensure diversity among the weak learners, we introduce the random selection of statistical invariants. That is, given a predicate set 
	$ \Psi=\{\psi_{\mathcal{T}_1},\,...,\,\psi_{\mathcal{T}_U},\,\psi_{\mathcal{S}_{U+1}},...,\,\psi_{\mathcal{S}_{(U+NV)}}\},
	$ any $\psi^*(\boldsymbol{x})$ can be chosen stochastically in each iteration to form the weak learner. As long as the given set of predicates includes meaningful knowledge, the increasing number of weak learners will ensure that available knowledge is utilized in constructing the ensemble learner. \par

\end{itemize}
%				Finally, we introduced an additional  \textbf{Stochastic Probability} sampling based on a uniform distribution on top of the bootstrapping process. This approach aims to further enhance the randomness of the training samples, aligning with the sample update mechanisms used in other related transfer learning ensemble methods. A detailed  description of the proposed extension is given in Algorithm \ref{alg:Al1}. \par 

By introducing statistical invariants and stochastic techniques, SETrLUSI effectively extracts various kinds of knowledge to enhance the stability of the learner. A detailed description of SETrLUSI is provided in Algorithm \ref{alg:Al1}.\par

%Additionally, converting source domain data into equality constraints rather than using them directly in the training process and eliminating the iteration of samples from both domains significantly also reduce the training time cost, which will  discussed in detail in the other section. 
%				\begin{small}
	%				\end{small}

%				\begin{figure}
	%					\centering
	%					\includegraphics[height=0.35\textheight]{TELUSI.pdf}
	%					\caption{A simple illustration of SETrLUSI. }
	%					\label{illustration}
	%				\end{figure}				

\subsection{Problem solving}
%	To solve the optimization problem \eqref{SETrLUSI-one}, we apply the squared loss in \eqref{frd}. Then, the objective function of \eqref{SETrLUSI-one} can be rewritten as
%	\begin{equation}\label{L2}
%		\underset{f}{\text{min}}~\int \left(\int 		\boldsymbol{\theta}(\boldsymbol{x}-\boldsymbol{\hat{x}})f(\boldsymbol{\hat{x}})\text{d}P(\boldsymbol{\hat{x}})-\int\boldsymbol{\theta}(\boldsymbol{x}-\boldsymbol{\hat{x}})\text{d}P(y=1,\boldsymbol{\hat{x}})\right)^2\sigma(\boldsymbol{x})\text{d}\mu(\boldsymbol{x}).
%	\end{equation}
	Since $P(\boldsymbol{x})$ and $P(y = 1, \boldsymbol{x})$ are unknown, we first approximate them by their empirical estimates $	\frac{1}{q}\sum_{k=1}^{q}\theta(\boldsymbol{x}-\boldsymbol{x}_k)$ and  $\frac{1}{q}\sum_{k=1}^{q}y_k\theta(\boldsymbol{x}-\boldsymbol{x}_k)$ \cite{vapnik2015v}, respectively. To control the model complexity, the regularization $\lVert f(\boldsymbol{x})\rVert^2$ is also considered. By applying the squared loss, the approximation of \eqref{SETrLUSI-one} is
	\begin{equation}\label{constraint}
		\begin{aligned}
			&\underset{f}{\text{min}}~\int \left(\sum_{k=1}^{q}  		\boldsymbol{\theta}(\boldsymbol{x}-\boldsymbol{x}_k)f(\boldsymbol{x}_k)-\sum_{k=1}^{q}y_k\boldsymbol{\theta}(\boldsymbol{x}-\boldsymbol{x}_k)\right)^2\sigma(\boldsymbol{x})\text{d}\mu(\boldsymbol{x})+\lambda\lVert f(\boldsymbol{x})\rVert^2\\
			&\mathrm{s.t.}~\frac{1}{q}\sum_{k=1}^q\psi^*(\boldsymbol{x}_k) f\left( \boldsymbol{x}_k \right) =\frac{1}{q}\sum_{k=1}^qy_k\psi^*(\boldsymbol{x}_k),
%			&\frac{1}{q}\sum_{k=1}^q{\psi_{\mathcal{T}_u}(\boldsymbol{x}_k) f\left( \boldsymbol{x}_k \right) =}\frac{1}{q}\sum_{k=1}^q{y_k\psi_{\mathcal{T}_u}(\boldsymbol{x}_k)},\ u=1,...,U,\\
%			&\frac{1}{q}\sum_{k=1}^q{\psi_{\mathcal{S}^i_v}(\boldsymbol{x}_k) 	f\left( \boldsymbol{x}_k \right) =}\frac{1}{q}\sum_{k=1}^q{y_k\psi_{\mathcal{S}^i_v}(\boldsymbol{x}_k)},~i=1,...,N,\ v=1,...,V.
		\end{aligned}
	\end{equation}
	where $\lambda>0$ is the trade-off parameter. We are looking for the solution in the Reproducing Kernel Hilbert Space associated with a kernel $\mathcal{K}(\boldsymbol{x})=(K(\boldsymbol{x}, \boldsymbol{x}_1),...,K(\boldsymbol{x}, \boldsymbol{x}_q))^\top$. Write  $\boldsymbol{F}(f)=(f(\boldsymbol{x}_1),...,f(\boldsymbol{x}_q))^\top=\boldsymbol{A}+b\boldsymbol{1}_q$, $\boldsymbol{A}=(a_1,...,a_d)^\top$, $\lVert f(\boldsymbol{x})\rVert^2=\boldsymbol{A}^{\top}\boldsymbol{KA}$, $\boldsymbol{K}=(\mathcal{K}(\boldsymbol{x}_1),...,\mathcal{K}(\boldsymbol{x}_q))^\top$, and $\boldsymbol{\psi}^*(\boldsymbol{x})=(\psi^*(\boldsymbol{x}_1),...,\psi^*(\boldsymbol{x}_q))^\top$. Consider the case where  $\sigma(\boldsymbol{x})=1$ and the distribution function $\mu(\boldsymbol{x})$ is the uniform distribution, that is $\mu(\boldsymbol{x})=\prod_{i=1}^d\frac{1}{q}\sum_{k=1}^{q}\theta(\boldsymbol{x}^i-\boldsymbol{x}_k^i)$ \cite{vapnik2015v}, then the optimization problem of \eqref{constraint} is rewritten as
	\begin{equation}\label{SETrLUSI}
		\begin{aligned} &\underset{\boldsymbol{A},\,b}{\min}~
			\left(\boldsymbol{F}(f)-\boldsymbol{Y}\right)^\top 	\boldsymbol{V} \left(\boldsymbol{F}(f)-\boldsymbol{Y}\right)+\lambda  \boldsymbol{A}^{\top}\boldsymbol{KA} \\
			&\mathrm{s.t.}~\boldsymbol{\psi}^*(\boldsymbol{x})^\top \boldsymbol{F}(f) =\boldsymbol{\psi}^*(\boldsymbol{x})^\top \boldsymbol{Y},
%			&\mathrm{s.t.}~ 					\frac{1}{q}\sum_{k=1}^q{\psi_{\mathcal{T}_u}(\boldsymbol{x}_k) f\left( \boldsymbol{x}_k \right) =}\frac{1}{q}\sum_{k=1}^q{y_k\psi_{\mathcal{T}_u}(\boldsymbol{x}_k)},\ u=1,...,U,\\
%			&~~~~~~\frac{1}{q}\sum_{k=1}^q{\psi_{\mathcal{S}^i_v}(\boldsymbol{x}_k) f\left( \boldsymbol{x}_k \right) =}\frac{1}{q}\sum_{k=1}^q{y_k\psi_{\mathcal{S}^i_v}(\boldsymbol{x}_k)},~i=1,...,N,\ v=1,...,V.
		\end{aligned}
	\end{equation}
	where $\boldsymbol{V}(\boldsymbol{x}_i,\,\boldsymbol{x}_j)=\int\boldsymbol{\theta}(\boldsymbol{x}-\boldsymbol{x}_i)\boldsymbol{\theta}(\boldsymbol{x}-\boldsymbol{x}_j)\text{d}(\boldsymbol{x})$. By writing $\boldsymbol{P}=\psi ^*(\boldsymbol{x})\psi^{*}(\boldsymbol{x})^\top$,  we further consider its unconstrained least squared formulation
	as follows:
	\begin{equation}\label{LUSI2}
		\begin{aligned}					\underset{\,\,\boldsymbol{A},b}{\min}~
			\left(\boldsymbol{F}(f)-\boldsymbol{Y}\right)^\top
			\left(\hat{\tau}\boldsymbol{ 	V}+\tau\boldsymbol{P}\right) \left(\boldsymbol{F}(f)-\boldsymbol{Y}\right)+\lambda  \boldsymbol{A}^{\top}\boldsymbol{KA},
		\end{aligned}
	\end{equation}
	where $\hat{\tau}+\tau=1$, $0<\hat{\tau}, \tau<1$ are hyperparameters that describe the relative importance of $\boldsymbol{V}$ and $\boldsymbol{P}$. Denote the objective function of \eqref{LUSI2} as $Q(\boldsymbol{A},b)$ and $\hat{\boldsymbol{P}}=\hat{\tau} \boldsymbol{V}+\tau \boldsymbol{P}$.
	From the necessary optimality conditions of minima of \eqref{LUSI2}, it has
	$\frac{\partial Q(\boldsymbol{A},b)}{\partial 	\boldsymbol{A}}=\boldsymbol{0},~\frac{\partial Q(\boldsymbol{A},b)}{\partial b}=0$,
%\begin{equation}
%		\frac{\partial Q(\boldsymbol{A},b)}{\partial 	\boldsymbol{A}}=\boldsymbol{K}\hat{\boldsymbol{P}}(\boldsymbol{KA}+b
%		\boldsymbol{1}_q-\boldsymbol{Y})+\lambda\boldsymbol{KA}=0,~\frac{\partial Q(\boldsymbol{A},b)}{\partial 	b}=\boldsymbol{1}^{\top}_q\hat{\boldsymbol{P}}(\boldsymbol{KA}+b\boldsymbol{1}_q-\boldsymbol{Y})=0.
%\end{equation}
	which follows
	\begin{equation}
		\boldsymbol{A}=(\hat{\boldsymbol{P}}\boldsymbol{K}+\lambda\boldsymbol{V})^{-1}
		\hat{\boldsymbol{P}}(\boldsymbol{Y}-b\boldsymbol{1}_q),
		~~b=\frac{\boldsymbol{1}_{q}^{\top}\boldsymbol{\hat{P}}
			(\boldsymbol{Y}-\boldsymbol{KA}_1)}
		{\boldsymbol{1}_{q}^{\top}\boldsymbol{\hat{P}}
			(\boldsymbol{1}_q-\boldsymbol{KA}_2)},
	\end{equation}
	where $\boldsymbol{A}_1=(\hat{\boldsymbol{P}}\boldsymbol{K}
	+\lambda\boldsymbol{V})^{-1}\hat{\boldsymbol{P}}\boldsymbol{Y},
	\boldsymbol{A}_2=(\hat{\boldsymbol{P}}\boldsymbol{K}
	+\lambda\boldsymbol{V})^{-1}\hat{\boldsymbol{P}}$. Then, we obtain the conditional probability function as $f(\boldsymbol{x})=P(y=1|\boldsymbol{x})=\boldsymbol{KA}+b$. For a given test sample $\boldsymbol{x}^*$ from the target domain, the conditional probability $f(\boldsymbol{x}^*)=P(y=1|\boldsymbol{x}^*)$ of $\boldsymbol{x}^*$ belonging to class 1 is obtained. If $P(y=1|\boldsymbol{x}^*)\ge0.5$, it is classified as class 1; otherwise, it is classified as class 0.

	\begin{algorithm}[h]
		\caption{Stochastic Ensemble Multi-Source Transfer Learning Using Statistical Invariant}\label{alg:Al1}
		\begin{algorithmic}[1]
			\STATE \textbf{Input} The maximum number of iterations $H$, source domain data $D_{\mathcal{S}^i}$, $i=1,...,N$, target domain training data $D_{\mathcal{T}^l}$, ratio $\gamma$, predicate set $\Psi=\{\psi_{\mathcal{T}_1},\,...,\,\psi_{\mathcal{T}_U},\,\psi_{\mathcal{S}_{U+1}},...,\,\psi_{\mathcal{S}_{(U+NV)}}\}$.\\
			%		where the first 
			%		$U$ come from target domain, and the remaining are from source domains.\\
			\STATE \textbf{Initialize}  An empty set of weak classifiers $\mathcal{F}^0=\emptyset$.\\
			\STATE \textbf{For} $h=1,..., H$ \textbf{do}\\
			\STATE \quad An empty set of candidate weak classifiers in the $h$-th round $\mathcal{F}_{sub}^h=\emptyset$. \\
			\STATE \quad \textbf{For} $i=1,..., N$ \textbf{do} \\
			\STATE \quad \quad \textbf{Stochastic target domain}: Obtain $D_{\mathcal{T}^l}^{h}$ from $D_{\mathcal{T}^l}$ by bootstrapping. \\
			\STATE \quad \quad \textbf{Stochastic source domain}: Obtain $D_{\mathcal{S}^i}^{h}$ from $D_{\mathcal{S}^i}$ with ratio $\gamma$.\\
			\STATE \quad \quad \textbf{Stochastic statistical invariant}: Select $\psi^*$ from  $\Psi$ and calculate based on  $D_{\mathcal{S}^i}^{h}$ or $D_{\mathcal{T}^l}^{h}$. \\
			%						\STATE \quad \quad \textbf{Stochastic Probability}: $0<p_{ik}^h\le1,\,k=1,...,q$,\\ ~~~~~~providing by a known distribution\\
			\STATE \quad \quad Learn $f^h$: $\boldsymbol{X}_{\mathcal{T}^l}\rightarrow \boldsymbol{Y}_{\mathcal{T}}$ with  $\psi^*$ by \eqref{STLUSI-ori} over the half of the set  $D_{\mathcal{T}^l}^{h}$. \\
			\STATE \quad \quad Compute the error $\epsilon_i^h$ of $f_i^h$
			$$
			\epsilon_i^h=\sum_{k=1}^{q}\frac{\{\textup{sign}\left(f^h_i\left( \boldsymbol{x}_{k}\right)-0.5\right)\neq y_{k}\}}{q}
			$$
			\STATE \quad \quad Set $\beta_i^h=1-\epsilon_i^h/(1-\epsilon_i^h)$.
			\STATE \quad \quad \textbf{If} $\epsilon_{i}^h\le0.5$ \textbf{then}
			$$\mathcal{F}_{sub}^h\rightarrow\mathcal{F}_{sub}^{h-1}\cup (f^h_{i},\,\beta^h_{i},\,\epsilon_{i}^h)$$\\
			\STATE \quad \quad \textbf{Else}
			$$\mathcal{F}_{sub}^h\rightarrow\mathcal{F}_{sub}^{h-1}$$\\
			\STATE \quad \quad \textbf{End}\\
			\quad \textbf{End}\\
			$\mathcal{F}^h\rightarrow\mathcal{F}^{h-1}\cup (f^h_{min},\,\beta^h_{min})$ from $\mathcal{F}_{sub}^h$ with the minima $\epsilon_{i}^h$.\\
			\textbf{End}
			\STATE $\hat{\beta}^{h}=\beta^h/\sum_{h=1}^{H}\beta^h$.
			\STATE \textbf{Output} the decision function
			$$
			f\left( \boldsymbol{x} \right) =\left\{ \begin{array}{l}
				1,~\sum_{h=1}^H{\hat{\beta}^hf^h\left( \boldsymbol{x} \right) \ge 0.5,}\\
				0,~\sum_{h=1}^H{\hat{\beta}^hf^h\left( \boldsymbol{x} \right) <0.5.}\\
			\end{array} \right. 
			$$
			
		\end{algorithmic}
	\end{algorithm}	
	\subsection{Theoretical analysis} 
	
	\subsubsection{Ensemble performance analysis}
	%				\subsubsection{Superiority analysis}
	%				First, we will present evidence demonstrating that the ensemble learner $f(\boldsymbol{x})$ outperforms individual base learners $f_h(\boldsymbol{x})$. \par
	%				
	%				\subsubsection{Convergence analysis}
	Denote $f^h(\boldsymbol{x})=P^h(y=1|\boldsymbol{x})$ as the conditional probability of an input $\boldsymbol{x}$ belonging to class 1 from the $h$-th weak learner, and weight of $f^h(\boldsymbol{x})$ as $\hat{\beta}^{h}=\beta^h/\sum_{h=1}^{H}\beta^h$, where $\beta^h=1-\epsilon_i^h/(1-\epsilon^h)\in [0,1]$ with $\epsilon^h\in[0,0.5)$ being the error of $f^h(\boldsymbol{x})$, $h=1,...,H$. According to our algorithm, the weighted ensemble learner is obtained as $\hat{f}(\boldsymbol{x}) =\sum_{h=1}^{H}\hat{\beta}^{h}f^h(\boldsymbol{x})$.   
	%				
	%				, the Bayes predictor gives the highest attainable correct classification rate
	%				\begin{equation}
		%					\begin{aligned}
			%						r^*=\int   \underset{j}{\textup{max}} P(j|\boldsymbol{x})\textup{d}P(\boldsymbol{x})
			%					\end{aligned}.
		%				\end{equation}
	Let $P^*(j|\boldsymbol{x})$ be the true conditional 
	probability that an input $\boldsymbol{x}$ is from class $j$. It can be proven that the performance of the ensemble leaner getting from SETrLUSI surpasses that of a weak leaner. 
	
	%				\begin{equation}
		%					\int\sum_{j=1}^{2} Q_h(j|\boldsymbol{x}) P(j|\boldsymbol{x})\textup{d}P(\boldsymbol{x})
		%				\end{equation}
	%				and 
	%				\begin{equation}
		%					\int_{\boldsymbol{x}\in C'}\sum_{j=1}^{2} \sum_{h=1}^{H}\hat{\beta}_{h}Q_h(j|\boldsymbol{x}) P(j|\boldsymbol{x})\textup{d}P(\boldsymbol{x}).
		%				\end{equation}
	%				Letting $C$ is the set of inputs $\boldsymbol{x}$ at which satisfy $\textup{argmax}_j Q(j|\boldsymbol{x})=\textup{argmax}_j P(j|\boldsymbol{x})$ and $C'$ is made by the remaining inputs $\boldsymbol{x}$, 
	%				we get for the correct 
	%				classification probability of SETrLUSI the expression
	%				\begin{equation}
		%					\begin{aligned}
			%						r^M=&\int_{\boldsymbol{x}\in C} \underset{j}{\textup{max}} P(j|\boldsymbol{x})\textup{d}P(\boldsymbol{x})+\\
			%						&\int_{\boldsymbol{x}\in C'}\sum_{j=1}^{2} \sum_{h=1}^{H}\hat{\beta}_{h}Q_h(j|\boldsymbol{x}) P(j|\boldsymbol{x})\textup{d}P(\boldsymbol{x}).
			%					\end{aligned}
		%				\end{equation}
	
	\begin{proposition}\label{pro1}
%		Assume a sample set $\{(\boldsymbol{x}_k,y_k)\}^m_{k=1}$, where $\boldsymbol{x}\in \boldsymbol{X}$ and $y\in\boldsymbol{Y}\in\{0,1\}$, drawn from a random
%		variable $D$ following the joint distribution $ P^*(\boldsymbol{X}, \boldsymbol{Y})$. 
	By implementing Algorithm \ref{alg:Al1}, without considering noise, for the ensemble learner $\hat{f}(\boldsymbol{x})=\sum_{h=1}^{H}\hat{\beta}^{h}P^h(j|\boldsymbol{x})=\sum_{h=1}^{H}\hat{\beta}^{h}f^h(\boldsymbol{x})$, it has
		\begin{equation}
			(P^*(y|\boldsymbol{x})-\hat{f}(\boldsymbol{x}))^2\le	\mathbb{E}_{D_{\mathcal{T}^l}}[P^*(y|\boldsymbol{x})-f^h(\boldsymbol{x})]^2.
		\end{equation}
	\end{proposition}
	\begin{proof}
		The proof is given in the Appendix \ref{Some proofs}.
	\end{proof}

	\begin{remark}
		To ensure the ensemble learner $\hat{f}(\boldsymbol{x})$ remains a conditional probability between 0 and 1, the error $\epsilon^h$ of each individual weak learner should must not exceed 0.5, which prevents $\beta^h< 0$.
	\end{remark}
	\begin{remark}
		Proposition \ref{pro1} establishes that our SETrLUSI outperforms any individual learner in estimating the conditional probability. Moreover, as the variance $\textup{Var}(f^h(\boldsymbol{x}))$ of $f^h(\boldsymbol{x})$ increases, SETrLUSI tends to exhibit lower squared error compared to a single weak learner, demonstrating its robustness in handling variability.
	\end{remark}

	\subsubsection{Convergence analysis}
	Now we explore the convergence behavior of SETrLUSI. Denote the weighted error as $S_e=\sum_{h=1}^{H}\hat{\beta}^hf^h(\boldsymbol{x})$, 
	where $f^h(\boldsymbol{x})=P^h(y=1|\boldsymbol{x})$. Based on the decision function described in Algorithm  \ref{alg:Al1}, our goal is to minimize the difference between $S_e$ and the true label $y$, ideally ensuring this difference is less than 0.5. Therefore, we define the misclassification probability as $P(|S_e-y|\ge \frac{1}{2})$. Based on this, we can derive an upper bound on the probability that the classification error rate of SETrLUSI deviates from its expected value. Subsequently, the following result can be obtained.
	
	\begin{proposition}\label{pro2}
		Given $H$ independent weak learners and their corresponding decision values $f^h(\boldsymbol{x})$, $h=1,...,H$, then an upper bound on the misclassification probability of SETrLUSI deviates from its true value $y$ is obtained as follows:
		\begin{equation}\label{ineq}
			P(|S_e-y|\ge \frac{1}{2})\le 2\textup{exp}{\left(-\frac{1}{2\sum_{h=1}^{H}(\hat{\beta}^h)^2}\right)}.
		\end{equation}
		%where notation \textup{exp}$(k)$ denotes $e^k$.
	\end{proposition}
	\begin{proof}
		The proof is given in the Appendix \ref{Some proofs}.
	\end{proof}
	%\begin{proof}
	%	Assuming the independence of the weak learners, we can apply Hoeffding inequality \cite{hoeffding1994probability} to derive the following bound
	%	\begin{equation}\label{ineq2}
		%		P(|S_e-\mathbb{E}[S_e]|\ge \delta)\le 2\textup{exp}{\left(-\frac{2\delta^2}{\sum_{h=1}^{H}(\hat{\beta}^h)^2}\right)}.
		%	\end{equation}
	%	We further assume that the expectation $\mathbb{E}[f^h(\boldsymbol{x})]$ accurately reflects the conditional probability of the true label $y$ in the ideal case, which means 
	%	\begin{equation}
		%		\mathbb{E}\left[\sum_{h=1}^{H}\hat{\beta}^hf^h(\boldsymbol{x})\right]=\sum_{h=1}^{H}\hat{\beta}^h\mathbb{E}\left[f^h(\boldsymbol{x})\right]=y\sum_{h=1}^{H}\hat{\beta}^h=y.
		%	\end{equation}
	%	By setting 
	%	$\delta=\frac{1}{2}$, we obtain \eqref{ineq}.
	%\end{proof}	
	\begin{remark}	
		 According to Cauchy-Schwarz inequality, is has $\sum_{h=1}^{H} (\hat{\beta}^h)^2\in\left[\frac{1}{H},1  \right)$. Thus, the right-hand of \eqref{ineq} converges to 0 as $H\rightarrow \infty$ if each weak learner doesn't have a such good predictive ability.
	\end{remark}
	\subsection{Predicate construction}\label{Predicate}
		To fully utilize knowledge from the source and target domains, the following predicates and their corresponding statistical invariants are constructed for the SETrLUSI, as shown in Table \ref{SITab}. \par
		\begin{table}[h]				
			\centering
			\fontsize{5}{10}\selectfont
			\caption{The statistical invariants from source domains and target domain}
			\label{SITab}
			\begin{threeparttable}
				\begin{tabular}{c|c|c}
					\toprule
					\textbf{Domain} & \textbf{Statistical Invariant} & \textbf{Predicate construction}\\
					\midrule
					\multirow{6}{*}{\makecell{Source\\domains}} &  $\sum_{k=1}^{q}f_{\mathcal{S}^i}(\boldsymbol{x}_k)f(\boldsymbol{x}_k)=\sum_{k=1}^{q}f_{\mathcal{S}^i}(\boldsymbol{x}_k)y_k$ & $f_{\mathcal{S}^i}(\boldsymbol{x})=\boldsymbol{w}_{\mathcal{S}^i}^{\top}\boldsymbol{x}+b$ \\
					& $\sum_{k=1}^{q}g_{\mathcal{S}^i}(\boldsymbol{x}_k)f(\boldsymbol{x}_k)=\sum_{k=1}^{q}g_{\mathcal{S}^i}g(\boldsymbol{x}_k)y_k$ & $g_{\mathcal{S}^i}(\boldsymbol{x})=\textup{ sign}(f_{\mathcal{S}^i}(\boldsymbol{x}))$ \\
					&  $\sum_{k=1}^{q}x_k^{s} 		f(\boldsymbol{x}_k)=\sum_{k=1}^{q}x^{s}_k y_k$ & $s=\underset{l}{\textup{max}}~ |w^l_{\mathcal{S}^i}|,l=1,...,d$\\
					&  $	\sum_{k=1}^{q}\boldsymbol{x}^{s_1}_k\boldsymbol{x}_k^{(s_2)\top} f(\boldsymbol{x}_k)=\sum_{k=1}^{q}\boldsymbol{x}^{s_1}_k\boldsymbol{x}_k^{(s_2)\top} y_k$ & $s_1,s_2=\underset{l_1,l_2}{\textup{max}}~ |\boldsymbol{w}^l_\mathcal{S}|,l_1,l_2=1,...,d$\\
					&  $\sum\limits_{k_1=1}^{q}{f(\boldsymbol{x}_{k_1}) 				\sum\limits_{k_2=1}^nK(\boldsymbol{x}_{k_2},\boldsymbol{x}_{k_2(\mathcal{S}^i)})}
					=							\sum\limits_{k_1=1}^{q}{y_{k_1}\sum\limits_{k_2=1}^nK(\boldsymbol{x}_{k_1},\boldsymbol{x}_{k_2(\mathcal{S}^i)})}$ & $K(\boldsymbol{x}_{k_1},\boldsymbol{x}_{k_2(\mathcal{S}^i)}) =\sum\limits_{k_2=1}^ne \textasciicircum\left(-\frac{\lVert \boldsymbol{x}_{k_1}-\boldsymbol{x}_{k_2(\mathcal{S}^i)}\rVert ^2}{2\sigma ^2}\right)$\\
					\midrule
					\multirow{5}{*}{\makecell{Target\\domain}}& $\sum_{k=1}^{q}\bar{\boldsymbol{x}}_kf(\boldsymbol{x}_k)=\sum_{k=1}^{q}\bar{\boldsymbol{x}}_ky_k$ & $\bar{\boldsymbol{x}}_k=(\sum_{i=1}^{d}\boldsymbol{x}_k^i)/d$\\
					& $\sum_{k=1}^{q}\boldsymbol{x}_k^s 		f(\boldsymbol{x}_k)=\sum_{i=1}^{q}\boldsymbol{x}_k^sy_k$ & $s=\textup{Random}(i),i=1,...,d$\\
					& $	\sum_{k=1}^{q}\bar{\boldsymbol{x}}_k\bar{\boldsymbol{x}}_k^\top 		f(\boldsymbol{x}_k)=\sum_{k=1}^{q}\bar{\boldsymbol{x}}_k\bar{\boldsymbol{x}}_k^\top y_k$ & $\bar{\boldsymbol{x}}_k=(\sum_{i=1}^{d}\boldsymbol{x}_k^i)/d$\\
					& $		\sum_{k=1}^{q}\boldsymbol{x}^{s_1}_k\boldsymbol{x}_k^{(s_2)\top} 	f(\boldsymbol{x}_k)=\sum_{k=1}^{q}\boldsymbol{x}^{s_1}_k\boldsymbol{x}_k^{(s_2)\top} y_k$ & $s_1,s_2=\textup{Random}(i),i=1,...,d$\\
					& $\sum_{k=1}^{q}f(\boldsymbol{x}_k)=\sum_{k=1}^{q}y_k$ & $\psi(\boldsymbol{x}_k)=1$\\
					\bottomrule 
				\end{tabular}
			\end{threeparttable}
		\end{table}
		All the above predicates totally could generate $(n_{fs}+n_{gs}+d+d(d+1)/2+n_{kernel}+1+d+1+d(d+1)/2+1)$ different statistical invariants for a $D_{\mathcal{S}^i}$ and $D_{\mathcal{T}}$, $i=1,...,N$, where $n_{fs}$, $n_{gs}$, and $n_{kernel}$ are the number of parameters in $\psi(\boldsymbol{x})=f_S(\boldsymbol{x})$, $\psi(\boldsymbol{x})=g_S(\boldsymbol{x})$, and $\psi(\boldsymbol{x})=\sum_{k_2=1}^nK(\boldsymbol{x},\boldsymbol{x}_{k_2(\mathcal{S}^i)})$, respectively. This allows SETrLUSI to extract diverse types of knowledge from both the source and target domains, thereby increasing the diversity of weak learners and ultimately enhancing the stability and performance of the ensemble learner. We give a detailed description in the Appendix \ref{The detail description of predicates}. 

	\subsection{Time complexity}\label{time}
%		According to Algorithm \ref{alg:Al1}, the computational complexity of SETrLUSI is analyzed as follows: 1) The time complexity of line 6 is $\mathcal{O}(q)$. 2) The time complexity of line 7 is $\mathcal{O}(\gamma n_{S_i})$. 3) The time complexity of line 8 is $\mathcal{O}(1+\psi)$, which is the time complexity of calculating $\psi$. 4) The time complexity $\mathcal{O}(dp^2)$ in line 10 for a linear programming problem.
%		5) The time complexity of line 11 to line 13 is $\mathcal{O}(q+1+1)$ and line 18 is $\mathcal{O}(NlogN+1)$ for finding the minima weak learner.\par
		By converting source domain knowledge into SI and the sampling strategy in the source domains and target domain, SETrLUSI further reduces the training time cost. According to Algorithm \ref{alg:Al1}, the time complexity of SETrLUSI is $\mathcal{O}(HN(dq^2+T_\psi))$, where $T_\psi\le dq^2$ is the time complexity for calculating one selected predicate.
%		 In most cases, it has $d\ll q$. 
%		Therefore, the maximum time complexity for SETrLUSI is $\mathcal{O}(T_\psi)=\mathcal{O}(dq^2)$. When selecting from the predicate set with equal probability, the time complexity of SETrLUSI is $\mathcal{O}(HN(dq^2))$ in $H$ iterations. 
%		Specifically, when $d\gg q$, calculating predicates related to features will introduce a tolerable increase in time cost. 
		 Table \ref{timecompare} shows the comparison between SETrLUSI and other related methods. A detailed analysis of time complexity is provided in the Appendix \ref{Comparison of the Time Complexity of Related Transfer Learning Methods}.\par
		\begin{table}[H]
			\centering
			\fontsize{7}{11}\selectfont
			\caption{
				Time complexity of SETrLUSI and its related transfer learning methods}
			\label{timecompare}
			\begin{threeparttable}
				\begin{tabular}{c|cc}
					\toprule
					\textbf{Method} & \textbf{Weak classifier}  & \textbf{Time complexity} \\
					\hline
					Tradaboost \cite{dai2007boosting} & SVM & $\mathcal{O}(Hd(q+\sum_{i=1}^{N}n_{\mathcal{S}^i})^2)$ \\
					MSDTrAdaboost \cite{zhang2014instance} & SVM & $\mathcal{O}(Hd\sum_{i=1}^{N}(q+n_{\mathcal{S}^i}^2))$ \\
					MSTrAdaboost \cite{yao2010boosting}  & SVM & $\mathcal{O}(Hd\sum_{i=1}^{N}(q+n_{\mathcal{S}^i}^2))$ \\
					MHTLTrAdaboost \cite{wang2016hierarchical}  & SVM & $\mathcal{O}(Hd(q+n')^2)$ \\
					WMSTrAdaboost \cite{antunes2019weighted}  & SVM & $\mathcal{O}(H\sum_{i=1}^{N}(q+dn_{\mathcal{S}^i}^2))$ \\
					METL \cite{yang2020multi} & Softmax, SVM, DNN & $\mathcal{O}(H(d\sum_{i=1}^{N}(q+n_{\mathcal{S}^i})^3+n_{b}\sum_{l=1}^{L}n_ln_{l+1}))$ \\
					3SW\_MSTL \cite{bai2022three} & Logistic Regression & $\mathcal{O}(d\sum_{i=1}^{N^*}n_i)$ \\
					SETrLUSI & LUSI & $\mathcal{O}(HN(dq^2+T_\psi))$, $T_\psi\le dq^2$ \\
					\bottomrule
				\end{tabular}
			\end{threeparttable}
			%				\vspace{0.5em} % 添加空白以分隔表格和注释
			\parbox{0.8\textwidth}{\scriptsize $n'$: number of samples after clustering all $D_\mathcal{S}$ and $D_\mathcal{T}$\\
				$n_b$: the batch size in DNN\\
				$n_l$: the number of neurons in $l$-th layer\\
				$N^*$: the number of source domains chosen by Bellwether method \cite{krishna2018bellwethers}}
		\end{table}

\section{Experiments}

\subsection{Experimental setting}\label{Experiment setting}
In this experiment, six related methods are used for comparison: MDT \cite{zhang2014instance}, MST \cite{yao2010boosting}, MHT \cite{wang2016hierarchical}, WMT \cite{antunes2019weighted}, METL \cite{yang2020multi}, 3SW \cite{bai2022three}, and TrA \cite{dai2007boosting}. For better validating the effectiveness of SETrLUSI, we employed three types of data—\textbf{UCI}, \textbf{20 News}, and \textbf{VLSC Object Image Classification}. For each transfer learning task, 10\% of the labeled target domain samples are used as training data, while the remaining samples are used for prediction. All the methods are carried out on a desktop with Windows 10 OS, 64-bit, 32 GB of RAM, Intel i7-7700K CPU 4.20GHz, by Matlab 2021b. More details about the experimental setting and datasets can be found in the Appendices \ref{More details about experimental setting} and \ref{Data Description}.  \par

\subsection{Experimental results}	
\subsubsection{Performance analysis}\label{Performance analyze}
Table \ref{AC} presents the average accuracy (AC), standard deviation (Std), and running time (time) recorded in seconds (s) over ten experimental trials, where the underlined entries indicate the second-best results, while the bold entries represent the optimal results. As shown in Table \ref{AC}, SETrLUSI consistently ranks among the top methods in terms of AC with a lower Std, outperforming the other six methods. This demonstrates that the incorporation of the stochastic selection of predicates enables the learner to effectively integrate diverse knowledge from multiple source domains and the target domain. Meanwhile, it achieves a 2.85\% higher accuracy than 3SW, the second-fastest method, while reducing runtime by 12 times compared to WMT, the second-most accurate approach, demonstrating its efficiency and effectiveness.\par

Based on the results in Table \ref{AC}, we perform the subsequent  Friedman test \cite{demvsar2006statistical} and post-hoc Nemenyi test \cite{nemenyi1963distribution}. Let \( n_d \) and  \( n_{me} \)  represent the number of datasets and the number of methods, respectively, and let \( r_j \) be the average rank of the \( j \)-th method across all datasets, where \( j = 1, \dots, n_{me} \). In our case, \( n_d = 18 \), \( n_{me} = 8 \), and the critical value of \( F(7, 141) \) at significance level \( \alpha \) is 1.760. Following \cite{demvsar2006statistical}, it obtains \( \omega_F = 18.0276 > 1.7600 \), which rejects the null hypothesis that all methods perform equally. After conducting the Friedman test, the Nemenyi test in Table \ref{Rank} further demonstrates that SETrLUSI outperforms most of the methods.\par
\setcounter{table}{3}
\begin{table}[h]				
	\centering
	\fontsize{8}{7}\selectfont
	\caption{The Nemenyi test result}
	\label{Rank}
	\begin{threeparttable}
		\begin{tabular}{p{0.03cm}c|cccccccc}
			\toprule
			\multicolumn{2}{c|}{Data} &TrA & MSD & MST & MHTL & WMT & METL & 3SW & SETrLUSI \\
			\midrule
			\multicolumn{2}{c|}{Average rank} & 7.11  & 2.89  & 4.37  & 5.11  & 2.74  & 5.63 & 6.68  & \pmb{1.00}\\
			\multicolumn{2}{c|}{Rank difference} & 6.11 & 	1.89 &	3.37 &	4.11 &	1.74 &	4.63 &	5.68 &-\\
			\multicolumn{2}{c|}{CD value \cite{nemenyi1963distribution}} & 2.26 & 2.26& 2.26 & 2.26 & 2.26 & 2.26 & 2.26  &-\\
			\bottomrule 
		\end{tabular}
	\end{threeparttable}
\end{table}

%	\begin{figure}[H]
	%		\centering
	%		\includegraphics[width=0.6\linewidth]{NF}
	%		\caption{Nemenyi test chart}
	%		\label{NF}
	%	\end{figure}

\subsubsection{Influence of the number of labeled target domain samples}
In this subsection, we validate the stability of the proposed model under the variation of the number of labeled target domain samples on the Rice, Wave, and TWO Norm three UCI datasets. Specifically, we extract 70\%, 50\%, 30\%, and 10\% of the target domain data as the training data, with the remaining portion serving as the test data. Figure \ref{DFS} shows that while other methods exhibit significant fluctuations with varying training data proportions, SETrLUSI demonstrates greater stability, particularly as the proportion of the training data decreases and enhances the efficiency of training data utilization.

\begin{figure}[h]
	\centering	
	\begin{minipage}{0.32\linewidth}
		\centering
		\includegraphics[width=1\linewidth]{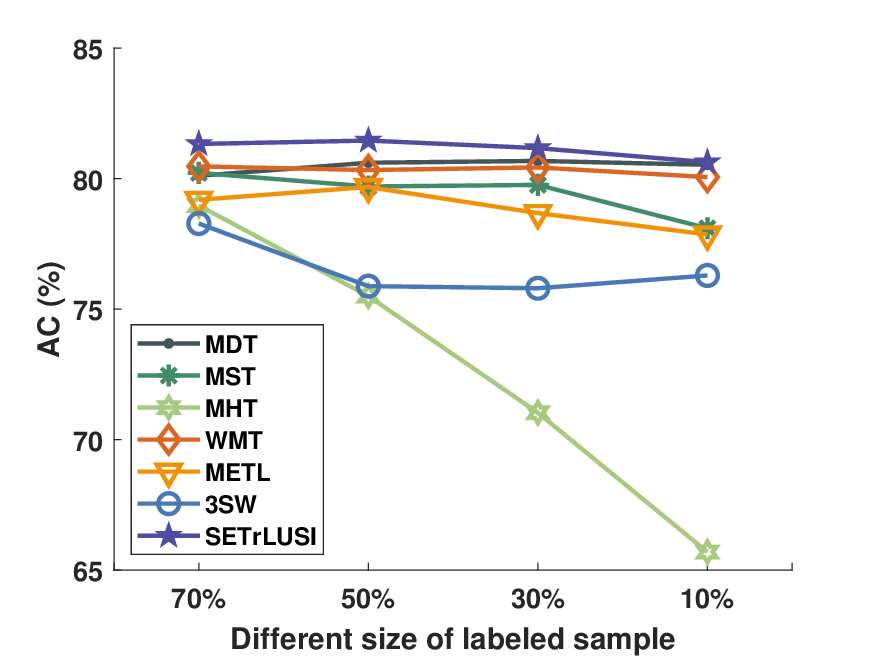}
		\caption*{(a) Rice 1\&2\&3}
	\end{minipage}
	\begin{minipage}{0.32\linewidth}
		\centering
		\includegraphics[width=1\linewidth]{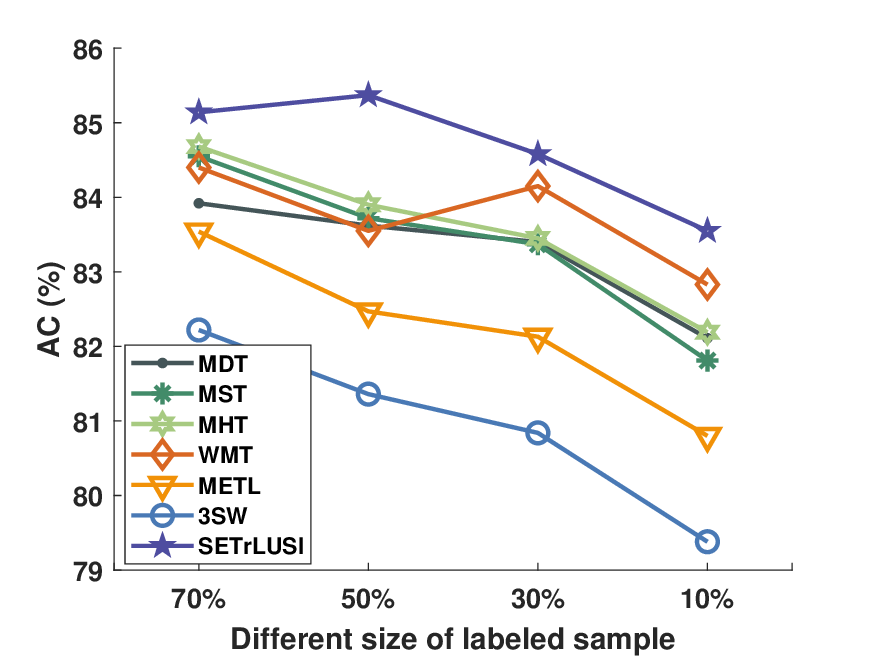}
		\caption*{(b) Wave 1\&7\&15}
	\end{minipage}
	\begin{minipage}{0.32\linewidth}
		\centering
		\includegraphics[width=1\linewidth]{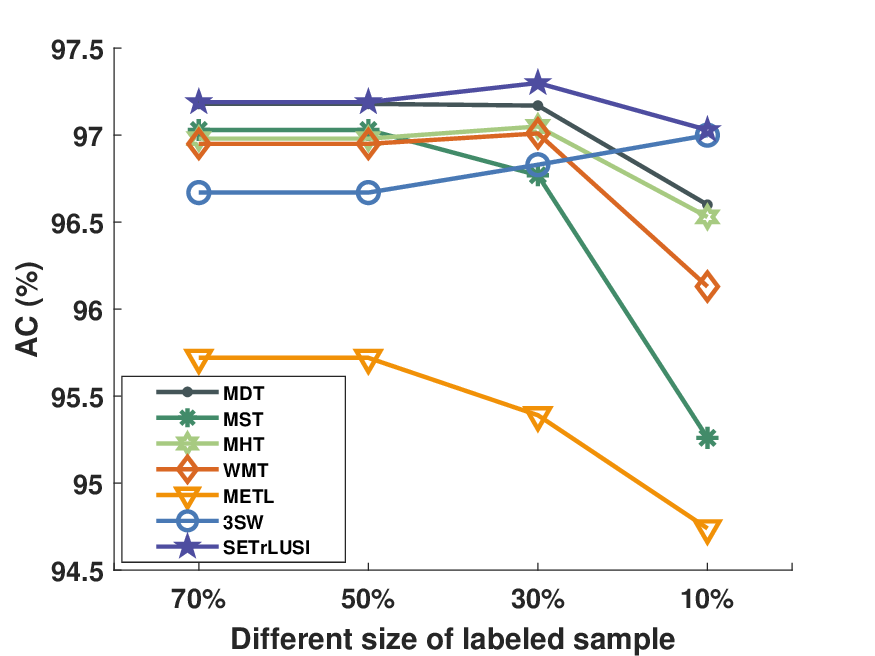}
		\caption*{(c) Two Norm 10\&13\&20}
	\end{minipage}
	\caption{The comparison on different size of the target domain training data on three UCI datasets, where the "1\&2\&3" means we use the first, second and third feature to cluster the three domains.}
	\label{DFS}
\end{figure}

\setcounter{table}{2}
\begin{table}[h]
	\centering
	\fontsize{6}{6.5}\selectfont
	\caption{The performance on all methods across 18 transfer learning tasks}
	\label{AC}
	\begin{threeparttable}
			\begin{tabular}{p{0.15cm}c|cccccccc}
					\toprule
					\multicolumn{2}{c|}{\multirow{3}{*}{Data}} &TrA & MDT & MST & MHT & WMT & METL & 3SW & SETrLUSI \\
					&& AC$\pm$Std &	AC$\pm$Std & AC$\pm$Std & AC$\pm$Std &	AC$\pm$Std & AC$\pm$Std & AC$\pm$Std &	AC$\pm$Std \\
					&& time(s) &	time(s) & time(s) &time(s) &time(s) &	time(s) &time(s) &time(s) \\
					\midrule
					\multirow{12}{*}{\rotatebox{90}{UCI}} & \multirow{2}{*}{Rice 1\&2\&3} &63.42$\pm$2.01 & \underline{80.53$\pm$0.57} & 78.10$\pm$1.73 & 65.69$\pm$8.76 & 80.06$\pm$0.43  &77.86$\pm$1.42 & 76.29$\pm$1.48 & \pmb{81.42$\pm$0.71} \\
					& & 20.72 &	7.86& 	8.35 &	\underline{4.17} &	18.30 &	467.14& 	5.50 &	\pmb{1.07}  \\
					&\multirow{2}{*}{Rice 3\&4\&7} &85.39$\pm$2.92 & 92.22$\pm$0.39 & 91.57$\pm$0.58 & 90.16$\pm$1.45 & 91.95$\pm$0.53  &\underline{92.28$\pm$0.54} & 90.46$\pm$1.11 & \pmb{92.77$\pm$0.39} \\
					&&16.84& 5.12 &	7.59 &	\underline{2.58} &	13.78& 	457.85 &	5.81 &	\pmb{1.31} \\
					%			&\multirow{2}{*}{rice134} &75.03$\pm$3.29 & \underline{83.51$\pm$0.78} & 81.94$\pm$0.85 & 77.57$\pm$0.94 & 82.66$\pm$0.51  &78.95$\pm$1.92 & 75.82$\pm$1.50 & \pmb{83.96$\pm$0.95} \\
					%			&&28.05 &	26.09& 	24.29 &	13.64& 	15.83 &	1758.56 &	5.07 &	1.08 \\
					&\multirow{2}{*}{Wave 1\&7\&15} &79.96$\pm$0.99 & 82.11$\pm$0.88 & 81.81$\pm$0.65 & 82.19$\pm$1.05 & \underline{82.83$\pm$0.89}  &80.80$\pm$0.79 & 79.38$\pm$0.57 & \pmb{84.13$\pm$0.82} \\
					&&223.51 &	23.15 &	77.50 &	11.46 &	89.38& 	1306.16 	&\underline{10.90} &	\pmb{1.89}\\ 
					&\multirow{2}{*}{Wave 1\&2\&4} &72.29$\pm$2.24 & 80.36$\pm$0.80 & 79.89$\pm$1.41 & 80.07$\pm$1.18 & \underline{82.70$\pm$1.50}  &78.99$\pm$0.92 & 78.10$\pm$0.60 & \pmb{83.20$\pm$0.43} \\
					&&222.92 &	26.41 &	69.86 &	\underline{9.02} &	80.84 &	739.82 & 10.70 &	\pmb{4.29} \\
					%			&\multirow{2}{*}{wave1719} &92.70$\pm$2.41 & 93.62$\pm$0.58 & 94.30$\pm$0.87 & 93.67$\pm$1.22 & 95.52$\pm$1.04  &93.01$\pm$0.95 & \pmb{96.52$\pm$0.23} & \underline{96.06$\pm$0.27} \\
					%			&&63.60 &	33.34 &	36.83 &	19.50 &	25.31 &	2788.91 &	9.80 &	4.90 \\
					&\multirow{2}{*}{TW 1\&3\&16} & 94.91$\pm$1.53 & 97.33$\pm$0.26 & 97.21$\pm$0.41 & 96.62$\pm$0.39 & 96.78$\pm$0.30 & \underline{97.34$\pm$0.28} & 96.13$\pm$0.19 & \pmb{97.52$\pm$0.27}\\
					&&59.05 &	\underline{26.04} &	31.35 &	60.34 &	47.37 &	1044.41 &	37.65 &	\pmb{3.85} \\					
					%			&\multirow{2}{*}{TW11520} & 95.74$\pm$0.74 & 96.87$\pm$0.32 & 95.76$\pm$0.98 & 96.49$\pm$0.44 & 96.25$\pm$0.62   &91.64$\pm$0.84 & \pmb{97.40$\pm$0.10} & \underline{97.08$\pm$0.26} \\
					%			&&91.08 &	32.94 &	37.72& 	13.1&	54.29 &	1233.13 &	40.75 &	4.14 \\
					&\multirow{2}{*}{TW 10\&13\&20} & 95.57$\pm$0.55 & 96.60$\pm$0.37 & 95.26$\pm$0.34 & 96.53$\pm$0.47 & 96.13$\pm$0.58 & 94.74$\pm$0.54 & \underline{97.00$\pm$0.25} & \pmb{97.24$\pm$0.24} \\
					&&79.34 &	30.82 &	29.68 &	\underline{9.81} &	44.92 &	1304.10 &	38.95 &	\pmb{4.09} \\
					\midrule
					%			\multicolumn{2}{c|}{\multirow{2}{*}{Average}}& 83.89$\pm$1.85 &	89.24$\pm$0.55 &	88.43$\pm$0.87 & 86.55$\pm$1.77	& \underline{89.43$\pm$0.71} &	87.29$\pm$0.91 &	87.46$\pm$0.67 &	\pmb{90.38$\pm$0.49} \\
					%			&&89.46 &	23.53 	&35.91 & \underline{15.96} &	43.33 &	1233.34 &	18.35 &	\pmb{2.96} \\
					%			\midrule
					\multirow{16}{*}{\rotatebox{90}{20News}}&\multirow{2}{*}{CS2R} &92.63$\pm$0.83 & 94.09$\pm$1.10 & 93.57$\pm$0.75 & 93.19$\pm$0.85 & \underline{95.86$\pm$0.39} & 93.15$\pm$0.41 & 91.59$\pm$0.51 & \pmb{95.57$\pm$0.44} \\
					&&370.06 &	439.37 &	217.47 &	43.86 &	380.57 &	1231.90 &	\underline{30.79} &	\pmb{12.72} \\
					&\multirow{2}{*}{CS2R} &89.17$\pm$1.12 & 90.60$\pm$1.36 & 89.97$\pm$1.95 & 89.62$\pm$1.54 & \underline{91.16$\pm$1.64} & 87.74$\pm$0.86 & 87.67$\pm$0.76 & \pmb{93.19$\pm$0.98} \\
					&&358.34 &	455.38 	&226.81 &	63.30 &	380.37& 	1891.21& 	\underline{37.58} &	\pmb{13.01} \\
					&\multirow{2}{*}{RS2C} &95.52$\pm$0.77 & 96.27$\pm$0.96 & 96.29$\pm$0.49 & 96.42$\pm$0.58 & \underline{97.24$\pm$0.40} & 95.74$\pm$0.36 & 95.20$\pm$0.36 & \pmb{98.16$\pm$0.34} \\
					&&350.43 &	429.50 &	169.56 &	41.68 &	347.95 	&1225.76 &	\underline{35.71} &	\pmb{13.09} \\
					&\multirow{2}{*}{CS2T} &94.71$\pm$0.80 & 95.17$\pm$0.93 & 95.23$\pm$0.67 & 95.15$\pm$0.43 & \underline{95.65$\pm$0.40} & 94.82$\pm$0.52 & 94.63$\pm$0.49 & \pmb{96.89$\pm$0.38} \\
					&&357.93 &	292.2&	104.77 &	\underline{34.42} &	395.00 	&1883.59 &	34.51 &	\pmb{12.50} \\
					&\multirow{2}{*}{CST2R} &92.26$\pm$0.88 & 94.24$\pm$0.65 & 93.68$\pm$0.88 & 93.27$\pm$0.91 & \underline{93.97$\pm$0.55} & 93.59$\pm$0.76 & 93.33$\pm$0.53 & \pmb{95.87$\pm$0.37} \\
					&&607.72 &	617.00 &	351.75 &	\underline{48.56} &	570.98& 	4714.33& 	80.40 &	\pmb{18.11}\\ 
					&\multirow{2}{*}{CRT2S} &88.99$\pm$1.18 & 90.81$\pm$1.57 & 90.21$\pm$1.50 & 89.57$\pm$1.48 & \underline{90.86$\pm$1.21}  & 88.04$\pm$0.89 & 88.24$\pm$0.71 & \pmb{93.16$\pm$1.02} \\
					&&670.19 &	788.76& 	395.69 &	72.36 &	608.18 &	2372.99 &	\underline{65.51} &	\pmb{18.37} \\
					&\multirow{2}{*}{RST2C} &95.22$\pm$1.00 & 96.50$\pm$0.42 & 96.19$\pm$0.51 & 96.38$\pm$0.69 & \underline{96.76$\pm$0.56} & 96.06$\pm$0.28 & 95.71$\pm$0.44 & \pmb{98.15$\pm$0.35} \\
					&&548.42 &	471.28& 	219.72& 	\underline{37.63}& 	129.57 	&2233.64 &	94.51 &	\pmb{16.21} \\
					&\multirow{2}{*}{CSR2T} &94.45$\pm$1.02 & 95.35$\pm$0.88 & 95.35$\pm$0.58 & 95.46$\pm$0.50 & \underline{96.13$\pm$0.45} & 95.01$\pm$0.59 & 94.98$\pm$0.52 & \pmb{96.90$\pm$0.38} \\
					&&740.35 &	599.32& 	292.48 &	\underline{57.10} &	554.09 &	4724.18 &	65.71 &	\pmb{15.91} \\
					\midrule						
					%			\multicolumn{2}{c|}{\multirow{2}{*}{Average}} & 92.87$\pm$0.95 &	94.13$\pm$0.98&	93.81$\pm$0.92&	93.63$\pm$0.87& \underline{94.43$\pm$0.74}&	93.02$\pm$0.58&	92.67$\pm$0.54&	\pmb{96.02$\pm$0.55}\\
					%			&&500.43 &	511.61 &	247.28 	& \underline{49.86 }&	420.84 	&2534.70 &	55.59 &	\pmb{14.99}  \\
					%			\midrule
					\multirow{8}{*}{\rotatebox{90}{VLSC}}&\multirow{2}{*}{SV2C} &97.83$\pm$0.49 & 98.61$\pm$0.81 & 97.32$\pm$0.57 & 97.13$\pm$0.48 & \pmb{100.00$\pm$0.00} & 96.67$\pm$1.15 & 97.97$\pm$0.32 & \pmb{100.00$\pm$0.00} \\
					&&3208.06 &	1700.10 &	2061.36 &	2929.66 &	235.95 &	5014.55 &	\pmb{18.11} &	\underline{125.87} \\
					%			&\multirow{2}{*}{LV2S}  &92.75$\pm$0.62 & \underline{93.16$\pm$0.26} & 92.65$\pm$0.48 & 92.96$\pm$0.67 & 92.64$\pm$0.47 & 90.98$\pm$0.45 &  93.25$\pm$0.24 & \pmb{93.44$\pm$0.33} \\
					%			&&2115.78 &	2768.97 &	748.35 &	2203.69 &	699.74 &	2634.19 &	8.08 &	104.23 \\
					&\multirow{2}{*}{SC2V} &92.97$\pm$0.92 & \underline{94.18$\pm$0.31} & 93.89$\pm$0.62 & 93.87$\pm$0.53 & 93.77$\pm$0.52 & 93.93$\pm$0.53 &  93.28$\pm$0.69 & \pmb{94.69$\pm$0.64} \\
					&&1712.18 &	2493.83 &	632.38 &	916.38 &	673.39 &	2386.01 &	\pmb{12.33} &	\underline{98.51} \\
					&\multirow{2}{*}{LSV2C} &99.47$\pm$0.39 & \pmb{100.00$\pm$0.00} & \pmb{100.00$\pm$0.00} & 99.23$\pm$0.57 & \pmb{100.00$\pm$0.00}   &95.82$\pm$1.33 & 99.23$\pm$0.22 & \pmb{100.00$\pm$0.00}\\
					&&24623.36 &	10197.09 	&5947.29 &	16494.90 &	2285.62 &	16798.23 	&318.58 &	\pmb{100.50} \\
					%			&\multirow{2}{*}{LCV2S} & 93.09$\pm$0.59 & \pmb{93.76$\pm$0.20} & 92.79$\pm$0.39 & 92.78$\pm$0.75 & 92.75$\pm$0.61 & 93.41$\pm$0.31 & 93.43$\pm$0.23 & \underline{93.54$\pm$0.27} \\
					%			&&12160.34 &	21235.69 &	5820.88 &	12088.90 &	7659.36 &	20347.12&	18.21 	&114.13 \\
					&\multirow{2}{*}{SCL2V} &93.73$\pm$0.67 & \underline{94.23$\pm$0.68} & 93.90$\pm$0.61 & 93.94$\pm$0.57 & 94.00$\pm$0.66  & 93.96$\pm$0.29 & 93.41$\pm$0.43 & \pmb{94.45$\pm$1.06} \\
					&&8675.97 &	8898.05 &	5366.40 &	4071.44 &	2404.45 &	20428.34 &	\pmb{110.86} &	\underline{139.41} \\
					\midrule
					\multicolumn{2}{c|}{\multirow{2}{*}{Average}} & 89.92$\pm$1.13&	92.73$\pm$0.72&	92.19$\pm$0.79 & 91.36$\pm$1.25&	\underline{92.98$\pm$0.63}&	91.47$\pm$0.69 & 91.26$\pm$0.57 & \pmb{94.09$\pm$0.53 }\\
					&&2380.30 &	1527.85 	&900.56 	&1383.81 &	514.48 &	3901.35 &	\underline{56.34} & \pmb{33.37}\\
					\bottomrule
				\end{tabular}
		\end{threeparttable}
%	\parbox{1\textwidth}{`s' is the second}
	%						\parbox{0.9\textwidth}{\scriptsize Note: bold indicates the second-best AC, while bold and underlined together denote the highest AC.}
\end{table}

\subsubsection{Convergence analysis}
To validate the convergence behavior of SETrLUSI, Figure \ref{20Con} clearly shows that SETrLUSI outperforms the other methods in terms of test error. In Figure \ref{20Con}, SETrLUSI exhibits a lower initial convergence point, indicating that the predicates effectively integrate knowledge from both the source and target domains into the learner's training process. 
%Additionally, it is evident that directly incorporating all predicates does not lead to improved generalization, confirming which confirms the effectiveness of the stochastic strategy.

\begin{figure}[h]
	\centering
	\begin{minipage}{0.245\linewidth}
		\centering
		\includegraphics[width=1\linewidth]{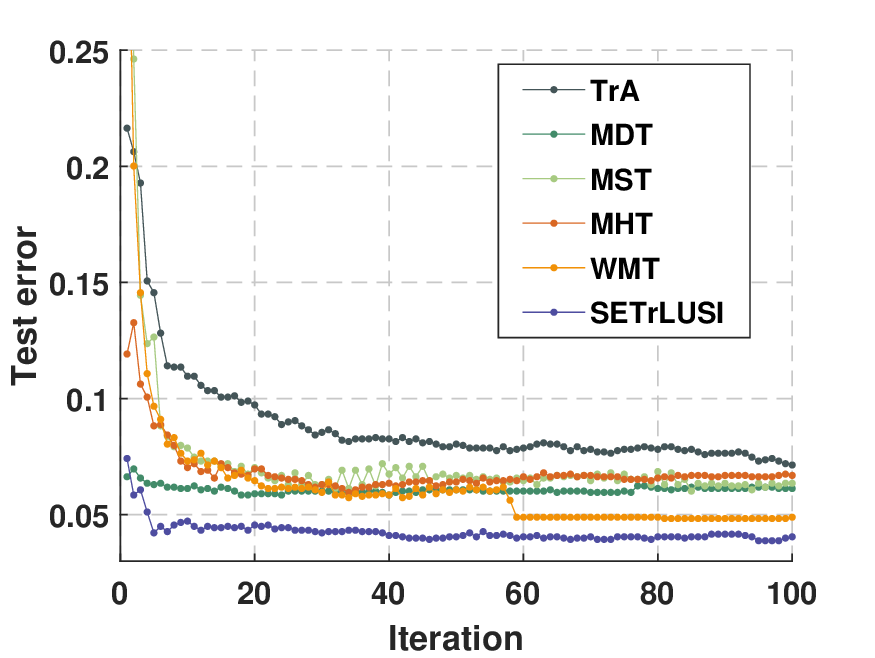}
		\caption*{(a) CS2R}
	\end{minipage}
	\begin{minipage}{0.245\linewidth}
		\centering
		\includegraphics[width=1\linewidth]{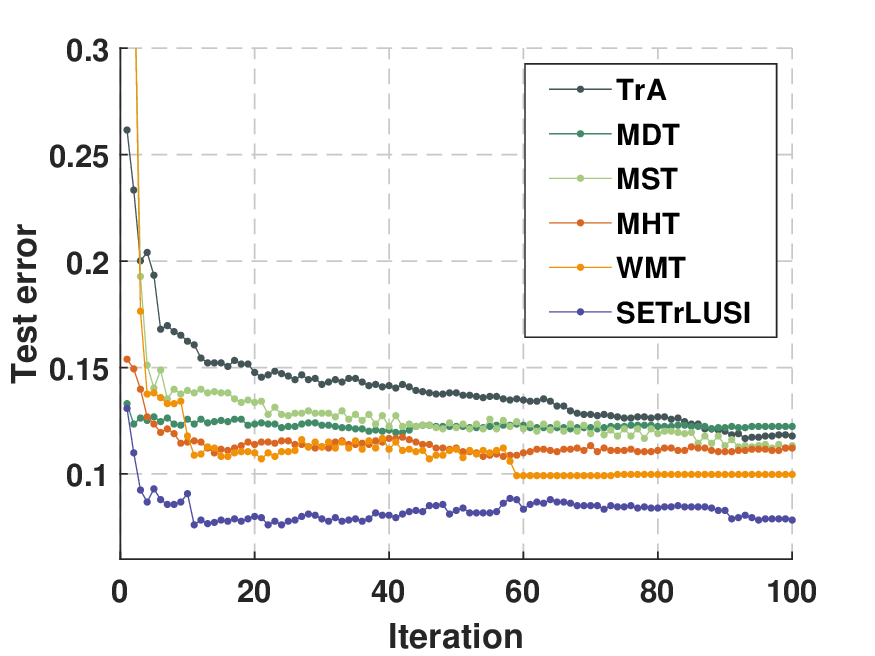}
		\caption*{(b) CR2S}
	\end{minipage}	
	\begin{minipage}{0.245\linewidth}
		\centering
		\includegraphics[width=1\linewidth]{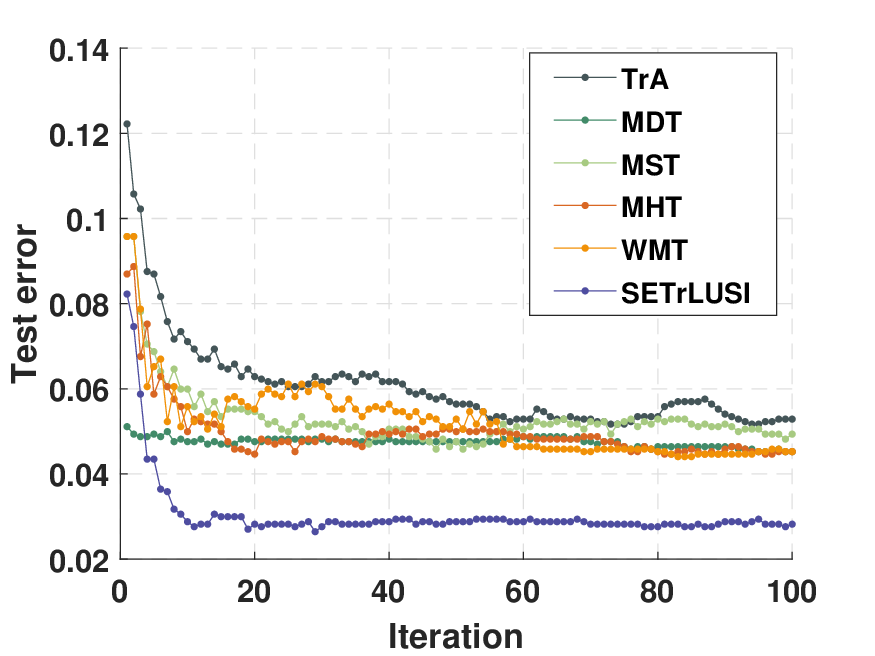}
		\caption*{(c) CSR2T}
	\end{minipage}
	\begin{minipage}{0.245\linewidth}
		\centering
		\includegraphics[width=1\linewidth]{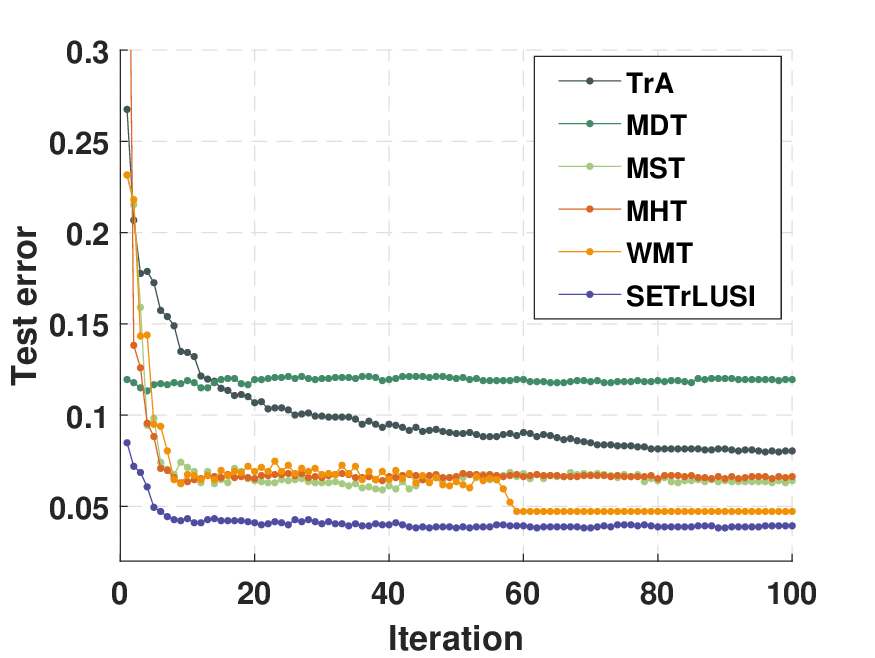}
		\caption*{(d) CST2R}
	\end{minipage}
	\caption{The convergence of six ensemble methods.}
	\label{20Con}
\end{figure}

%	\subsubsection{Time complexity}	
%	To validate the analysis of time complexity presented earlier, we recorded the time taken by all methods for 100 iterations and compared the time consumption.\par
%		
%	Figure \ref{time} illustrates the time consumption across three categories of data with varying dimensions, where the time for each category represents the average of all datasets, displayed on a logarithmic scale. It clearly shows that, on the lower-dimensional UCI datasets, SETrLUSI has minimal time consumption. However, as the dimensionality increases, MHTL may consume less time than SETrLUSI due to its time complexity being influenced by the number of clustered samples. Additionally, 3SW, utilizing logistic regression as its base learner, is less affected by dimensionality, resulting in its time consumption also being lower than that of SETrLUSI.
%	
%	\begin{figure}[H]
	%		\centering
	%		\includegraphics[width=0.7\linewidth]{timecost}
	%		\caption{The time consumption comparison across three types of data with different dimensions ($d_{UCI}\ll d_{20 News}\ll d_{VLSC}$).}
	%		\label{time}
	%	\end{figure}

\subsubsection{Source domain diversity analysis}
To validate whether SETrLUSI can effectively leverage source domain knowledge for the target domain learner across increasingly diverse source domains, twelve similar but different domains are constructed as in the left panel of Figure \ref{MT}. The experimental results in the right panel of Figure \ref{MT} show that SETrLUSI exhibits steady performance improvement as the number of source domains increases, consistently achieving higher AC compared to the baseline methods. The Std reveals that the stability of predictions improves with the inclusion of more source domains, with SETrLUSI demonstrating superior stability over the competing methods.

%		
%		as illustrated in the figure \ref{ddsd}, with the first domain designated as the target domain and the remains serving as source domains. Each source domain consists of 500 positive samples and 500 negative samples. \par
%		\begin{figure}[H]
	%			\centering
	%			\includegraphics[width=0.6\linewidth]{MT}
	%			\caption{The source diversity experiment.}
	%			\label{MT}
	%		\end{figure}
\begin{figure}[h]
	\centering
	\begin{minipage}{0.55\linewidth}
		\centering
		\includegraphics[width=1\linewidth]{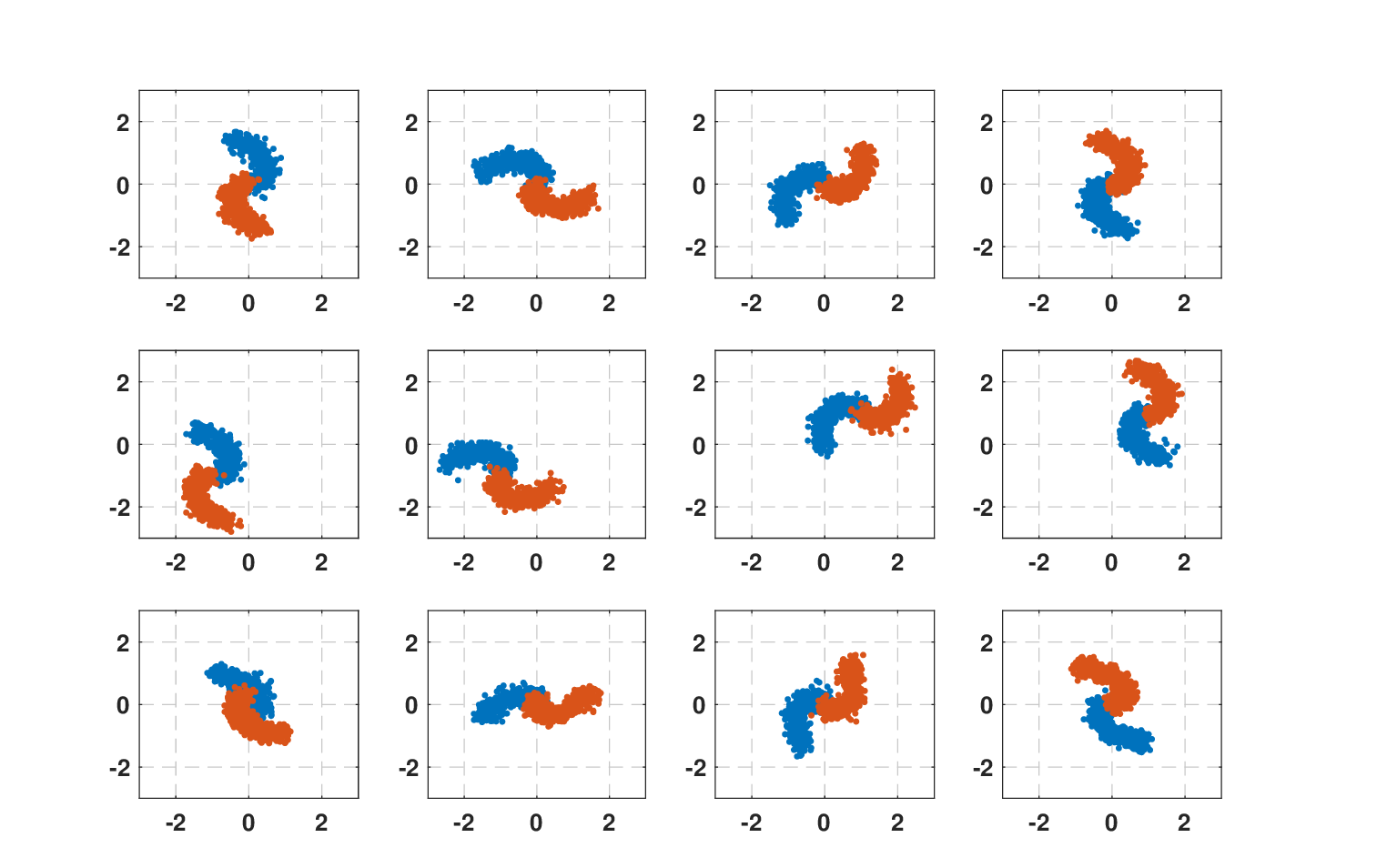}
		%				\caption*{(a) CS2R}
	\end{minipage}
	\begin{minipage}{0.44\linewidth}
		\centering
		\includegraphics[width=1\linewidth]{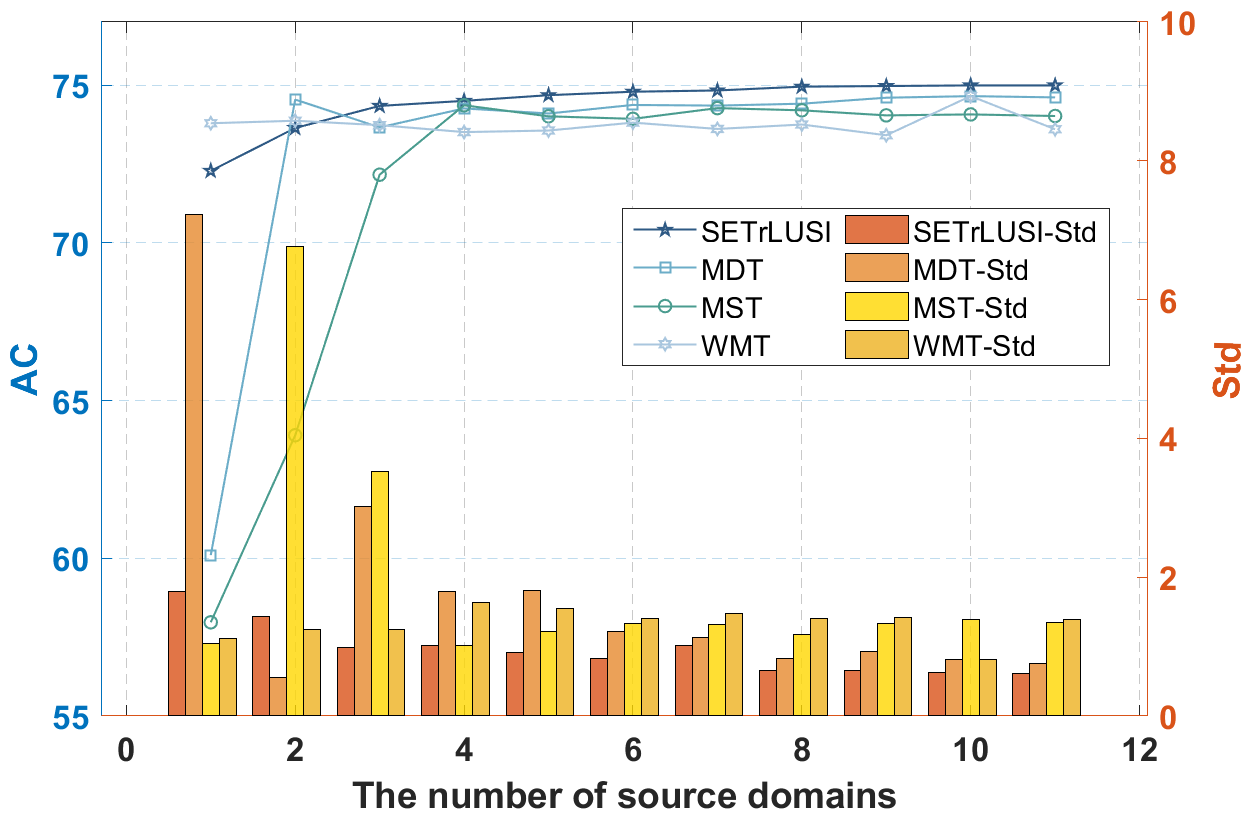}
		%				\caption*{(b) CR2S}
	\end{minipage}	
	\caption{The 12 domains that are constructed by adjusting the rotation angle, center position, and the compactness of classes (left panel) and the results of source domain diversity analyze (right panel).}
	\label{MT}
\end{figure}

\section{Conclusion}\label{Conclusion}
This paper proposes a novel multi-source domain knowledge transfer method, Stochastic Ensemble Multi-Source Transfer Learning Using Statistical Invariant (SETrLUSI). It integrates source and target domain knowledge into the target domain learner by statistical invariants as a form of the weak mode of convergence with stochastic selection, which improves the diversity of the weak learners and enhances the stability of the ensemble learner. To further utilize the source domains and limited target domain efficiently, SETrLUSI employs proportional sampling for the source domains and bootstrapping for the target domain, which is considered to reflect the original distribution of the target domain and reduce the model's time complexity. \par

However, the computational time of SETrLUSI is influenced by the form of predicates. For example, when a feature predicate is selected in high-dimensional data, it will result in a substantial time cost. Addressing this challenge will be a focus of our future work. The code of SETrLUSI can be found on Github \footnote{https://github.com/HDSYW/SETrLUSI}.

\newpage
\bibliographystyle{IEEEtran}	
\bibliography{cite}

%%%%%%%%%%%%%%%%%%%%%%%%%%%%%%%%%%%%%%%%%%%%%%%%%%%%%%%%%%%%

%%%%%%%%%%%%%%%%%%%%%%%%%%%%%%%%%%%%%%%%%%%%%%%%%%%%%%%%%%%%
\newpage
\appendix

\section{Appendices}
%Technical appendices with additional results, figures, graphs and proofs may be submitted with the paper submission before the full submission deadline (see above), or as a separate PDF in the ZIP file below before the supplementary material deadline. There is no page limit for the technical appendices.

\subsection{The proofs of Proposition 1 and Proposition 2}\label{Some proofs}
\setcounter{proposition}{0}
\begin{proposition}
%	Assume a sample set $\{(\boldsymbol{x}_k,y_k)\}^m_{k=1}$, where $\boldsymbol{x}\in \boldsymbol{X}$ and $y\in\boldsymbol{Y}\in\{0,1\}$, drawn from a random
%	variable $D$ following the joint distribution $ P^*(\boldsymbol{X}, \boldsymbol{Y})$. 
	By implementing Algorithm \ref{alg:Al1}, without considering noise, for the ensemble learner $\hat{f}(\boldsymbol{x})=\sum_{h=1}^{H}\hat{\beta}^{h}P^h(j|\boldsymbol{x})=\sum_{h=1}^{H}\hat{\beta}^{h}f^h(\boldsymbol{x})$, it has
\begin{equation}
	(P^*(y|\boldsymbol{x})-\hat{f}(\boldsymbol{x}))^2\le	\mathbb{E}_{D_{\mathcal{T}^l}}[P^*(y|\boldsymbol{x})-f^h(\boldsymbol{x})]^2.
\end{equation}
\end{proposition}
\begin{proof}
	Given the expected squared risk as $\mathbb{E}_{D_{\mathcal{T}^l}}[P^*(y|\boldsymbol{x})-f^h]^2$, over possible training set drawn from $D$, showing the following three terms decomposition
	\begin{equation}\label{P1}
		\mathbb{E}_{D_{\mathcal{T}^l}}[P^*(y|\boldsymbol{x})-f^h(\boldsymbol{x})]^2 =(P^*(y|\boldsymbol{x}))^2-2P^*(y|\boldsymbol{x})\mathbb{E}_{D_{\mathcal{T}^l}}[\hat{f}^h(\boldsymbol{x})]+\mathbb{E}_{D_{\mathcal{T}^l}}[(\hat{f}^h(\boldsymbol{x}))^2].
	\end{equation}
	By utilizing  $\textup{Var}(\hat{f}^h(\boldsymbol{x}))=\mathbb{E}_{D_{\mathcal{T}^l}}[(f^h(\boldsymbol{x}))^2]-\mathbb{E}^2_{D_{\mathcal{T}^l}}[f^h(\boldsymbol{x})]$, it gives 
	\begin{equation}
		\begin{aligned}
			&\mathbb{E}_{D_{\mathcal{T}^l}}[P^*(y|\boldsymbol{x})-f^h(\boldsymbol{x})]^2 =(P^*(y|\boldsymbol{x}))^2-2P^*(y|\boldsymbol{x})\mathbb{E}_{D_{\mathcal{T}^l}}[f^h(\boldsymbol{x})]+\textup{Var}(f^h(\boldsymbol{x}))+\mathbb{E}^2_{D_{\mathcal{T}^l}}[f^h(\boldsymbol{x})]\\
			&~~~~~~~~~~~~~~~~~~~~~~~~~~~~~~~~~~~~~~~~~~\ge(P^*(y|\boldsymbol{x}))^2-2P^*(y|\boldsymbol{x})\mathbb{E}_{D_{\mathcal{T}^l}}[f^h(\boldsymbol{x})]+\mathbb{E}^2_{D_{\mathcal{T}^l}}[f^h(\boldsymbol{x})]\\
			&~~~~~~~~~~~~~~~~~~~~~~~~~~~~~~~~~~~~~~~~~~\ge (P^*(y|\boldsymbol{x})-\mathbb{E}_{D_{\mathcal{T}^l}}[f^h(\boldsymbol{x})])^2.
		\end{aligned}
	\end{equation}
	Due to $\mathbb{E}_{D_{\mathcal{T}^l}}[f^h]=\frac{1}{H}\sum_{h=1}^{H}f^h$ overlooking the impact of differences in the performance of weak learners on the final outcome. Let us assume that the weak learners follow a certain probability distribution that depends on $\epsilon^h$, we have the formulation of expectation as
	\begin{equation}
		\mathbb{E}_{D_{\mathcal{T}^l}}[f^h(\boldsymbol{x})]=\hat{f}(\boldsymbol{x})=\sum_{h=1}^{H}\hat{\beta}^{h}f^h(\boldsymbol{x}),
	\end{equation}
	where $\hat{\beta}^{h}=\beta^h/\sum_{h=1}^{H}\beta^h$, $\beta^h=1-\epsilon^h/(1-\epsilon^h)$, which increases as the error of the weak learner decreases. Then, \eqref{P1} is proved.
\end{proof}
\begin{proposition}
	Given $H$ independent weak learners and their corresponding decision values $f^h(\boldsymbol{x})$, $h=1,...,H$, then an upper bound on the misclassification probability of SETrLUSI deviates from its true value $y$ is obtained as follows:
	\begin{equation}
		P(|S_e-y|\ge \frac{1}{2})\le 2\textup{exp}{\left(-\frac{1}{2\sum_{h=1}^{H}(\hat{\beta}^h)^2}\right)}.
	\end{equation}
%	where notation \textup{exp}$(k)$ denotes $e^k$.
\end{proposition}
\begin{proof}
	Assuming the independence of the weak learners, we can apply Hoeffding inequality \cite{hoeffding1994probability} to derive the following bound
	\begin{equation}\label{ineq2}
		P(|S_e-\mathbb{E}[S_e]|\ge \delta)\le 2\textup{exp}{\left(-\frac{2\delta^2}{\sum_{h=1}^{H}(\hat{\beta}^h)^2}\right)}.
	\end{equation}
	We further assume that the expectation $\mathbb{E}[f^h(\boldsymbol{x})]$ accurately reflects the conditional probability of the true label $y$ in the ideal case, which means 
	\begin{equation}
		\mathbb{E}\left[\sum_{h=1}^{H}\hat{\beta}^hf^h(\boldsymbol{x})\right]=\sum_{h=1}^{H}\hat{\beta}^h\mathbb{E}\left[f^h(\boldsymbol{x})\right]=y\sum_{h=1}^{H}\hat{\beta}^h=y.
	\end{equation}
	By setting 
	$\delta=\frac{1}{2}$, we obtain \eqref{ineq}.
\end{proof}

\subsection{The detail description of predicates}\label{The detail description of predicates}

\begin{itemize} 
	\item $\psi(\boldsymbol{x})=f_{\mathcal{S}^i}(\boldsymbol{x})$ : 
	This predicate integrates the performance of the information of source domain's model into the training process of the target domain model, leveraging the confidence interval information derived from the source domain as a weight. The corresponding SI is 
	\begin{equation}
		\sum_{k=1}^{q}f_{\mathcal{S}^i}(\boldsymbol{x}_k)f(\boldsymbol{x}_k)=\sum_{k=1}^{q}f_{\mathcal{S}^i}(\boldsymbol{x}_k)y_k, 
	\end{equation}			
	where $f_{\mathcal{S}^i}(\boldsymbol{x})=w_{\mathcal{S}^i}^\top \boldsymbol{x} +b_{\mathcal{S}^i}$ is a decision function trained from $D_{\mathcal{S}^i}$, which employs a standard linear SVM and selects the regularization parameter randomly to form the learner.
	
	\item $\psi(\boldsymbol{x})=g_{\mathcal{S}^i}(\boldsymbol{x})$ : This predicate gives a SI that expresses the ratio of correctly classified positive samples from the target domain by $f_{\mathcal{S}^i}(\boldsymbol{x})$ is equal to the expectation of samples where both $f_{\mathcal{S}^i}(\boldsymbol{x})$ and $f(\boldsymbol{x})$ correctly classify the samples.
	\begin{equation}
		\sum_{k=1}^{q}g_{\mathcal{S}^i}(\boldsymbol{x}_k)f(\boldsymbol{x}_k)=\sum_{k=1}^{q}g_{\mathcal{S}^i}g(\boldsymbol{x}_k)y_k, 
	\end{equation}
	where $g_{\mathcal{S}^i}(\boldsymbol{x})$ is constructed based on $f_{\mathcal{S}^i}(\boldsymbol{x})$ and obtained by taking the sign function of $f_{\mathcal{S}^i}(\boldsymbol{x})$.
	\begin{equation}
		g_{\mathcal{S}^i}(\boldsymbol{x}) 	=\left\{\begin{array}{l}
			0,\ \mathrm{if}\ 	f_{\mathcal{S}^i}\left( \boldsymbol{x}\right) \le 0,\\
			1,\ \mathrm{if}\ 	f_{\mathcal{S}^i}\left( 	\boldsymbol{x}\right) >0.\\
		\end{array} \right.
	\end{equation}
	
	\item $\psi(\boldsymbol{x})=\boldsymbol{x}^{s}_{\mathcal{S}^i}$ : This predicate conveys knowledge about the feature, but the feature it provides is obtained through correlation analysis in the source domain, as opposed to the direct random selection in the target domain.
	\begin{equation}
		\sum_{k=1}^{q}\boldsymbol{x}_{k({\mathcal{S}^i})}^{s} 		f(\boldsymbol{x}_k)=\sum_{k=1}^{q}\boldsymbol{x}^{s}_{k({\mathcal{S}^i})}y_k, 
	\end{equation}
	where we utilize the absolute value of the weight $\boldsymbol{w}$ from a linear SVM as the criteria for selecting the $s$-th feature that has the strongest relationship between $D_{\mathcal{S}^i}$ with $\boldsymbol{Y}_{\mathcal{S}^i}$  and gives a total of $d$ statistical invariants.\par
	
	\item $\psi(\boldsymbol{x})=\boldsymbol{x}^{s_1}_{\mathcal{S}^i}\boldsymbol{x}^{s_2}_{\mathcal{S}^i},~ s_1, s_2=1,...,d$ : Based on $\psi(\boldsymbol{x})=\boldsymbol{x}^{s}_{\mathcal{S}^i}$, we can also identify the two most correlated features $s_1$ and $s_2$ to form the SI.
	\begin{equation}
		\sum_{k=1}^{q}\boldsymbol{x}^{s_1}_{k(\mathcal{S})}\boldsymbol{x}_{k({\mathcal{S}^i})}^{s_2} f(\boldsymbol{x}_k)=\sum_{k=1}^{q}\boldsymbol{x}^{s_1}_{k({\mathcal{S}^i})}\boldsymbol{x}_{k({\mathcal{S}^i})}^{s_2} y_k,
	\end{equation}
	where this results in $\frac{d(d+1)}{2}$ possible statistical invariants.
	
	\item $\psi(\boldsymbol{x})=\sum\limits_{k_2=1}^nK(\boldsymbol{x},\boldsymbol{x}_{k_2(\mathcal{S}^i)})$ : This predicate incorporates sample knowledge from the source domain into the training of the target domain classifier by utilizing the optimal parameter derived from predicting the known samples from the target domain by source domains. 
	\begin{small}
		\begin{equation}
			\sum_{k_1=1}^{q}{f(\boldsymbol{x}_{k_1}) 				\sum\limits_{k_2=1}^nK(\boldsymbol{x}_{k_2},\boldsymbol{x}_{k_2(\mathcal{S}^i)})}
			=					\sum_{k_1=1}^{q}{y_{k_1}\sum\limits_{k_2=1}^nK(\boldsymbol{x}_{k_1},\boldsymbol{x}_{k_2(\mathcal{S}^i)})}, 
		\end{equation}
	\end{small}
	where $K(\boldsymbol{x}_{k_1},\boldsymbol{x}_{k_2(\mathcal{S}^i)}) =\sum\limits_{k_2=1}^ne^{-\frac{\lVert \boldsymbol{x}_{k_1}-\boldsymbol{x}_{k_2(\mathcal{S}^i)}\rVert ^2}{2\sigma ^2}}$ is the radial basis function kernel, $\sigma$ is the kernel parameter.	
	
	%	\item $\psi(\boldsymbol{x})=\boldsymbol{0}$: When the knowledge extracted from the source domain decreases the learning performance in the target domain, negative transfer occurs \cite{zhang2022survey}. Therefore, we additionally include the option of not adding predicates to the selection range, allowing RTLUSI to randomly choose from a broader set of possibilities.			
\end{itemize}

\begin{itemize}
	\item $\psi(\boldsymbol{x})=\bar{\boldsymbol{x}}$: This predicate provides an SI, which is expected to be the average of the center of mass of the prediction to be equal the true value of class 1 $(y=1)$ \cite{vapnik2020complete}. 
	\begin{equation}
		\sum_{k=1}^{q}\bar{\boldsymbol{x}}_k 			f(\boldsymbol{x}_k)=\sum_{k=1}^{q}\bar{\boldsymbol{x}}_ky_k,
	\end{equation}
	where $\bar{x}$ is obtained by summing each column of features and then taking the average.
	
	\item $\psi(\boldsymbol{x})=\boldsymbol{x}^s_{\mathcal{T}},~s=1,...,d$ : This predicate gives $d$ statistical invariants for randomly selecting any one feature to pursue the equal of the center of mass between the prediction $f(\boldsymbol{x})$ and true value $\boldsymbol{Y}$.
	\begin{equation}
		\sum_{k=1}^{q}\boldsymbol{x}_{k(\mathcal{T})}^s 		f(\boldsymbol{x}_{k})=\sum_{i=1}^{q}\boldsymbol{x}_{k(\mathcal{T})}^sy_k, 
	\end{equation}			
	where $s$ is determined randomly in each selection.
	
	\item $\psi(\boldsymbol{x})=\bar{\boldsymbol{x}}\cdot\bar{\boldsymbol{x}}$: This predicate eagles to maintain the average of the expectation of the covariance matrix computed using $f(\cdot)$ equal to the corresponding matrix computed from the true value $\boldsymbol{Y}$ by the following SI \cite{vapnik2020complete}
	\begin{equation}
		\sum_{k=1}^{q}\bar{\boldsymbol{x}}_k\bar{\boldsymbol{x}}_k^\top 		f(\boldsymbol{x}_k)=\sum_{k=1}^{q}\bar{\boldsymbol{x}}_k\bar{\boldsymbol{x}}_k^\top y_k, 
	\end{equation}
	where $\bar{x}$ is given as same as the one in predicate $\psi(\boldsymbol{x})=\bar{x}$.
	
	\item $\psi(\boldsymbol{x})=\boldsymbol{x}^{s_1}_{k(\mathcal{T})}\cdot \boldsymbol{x}^{s_2}_{k(\mathcal{T})},~s_1, s_2=1,...,d$ :
	To further utilize the relationship between the covariance matrix calculated from the prediction $f(\boldsymbol{x})$ and the true value $\boldsymbol{Y}$ for each dimension, this predicate will provide $\frac{d(d+1)}{2}$
	statistical invariants to expand the range of random selections.
	\begin{equation}\label{RamXX}
		\sum_{k=1}^{q}\boldsymbol{x}^{s_1}_{k(\mathcal{T})}\boldsymbol{x}_{k(\mathcal{T})}^{s_2} 	f(\boldsymbol{x}_k)=\sum_{k=1}^{q}\boldsymbol{x}^{s_1}_{k(\mathcal{T})}\boldsymbol{x}_{k(\mathcal{T})}^{s_2} y_k, 
	\end{equation}
	where $d$ is determined randomly in each selection.
	
	\item $\psi(\boldsymbol{x})=\boldsymbol{1}$: This predicate restricts the function $f(\cdot)$ to those for which the frequency of prediction $f(\boldsymbol{x})$ from  class 1 $(y=1)$ matches the frequency observed in the training data. The statistical invariant is made as follows \cite{vapnik2020complete}:
	\begin{equation}
		\sum_{k=1}^{q}f(\boldsymbol{x}_k)=\sum_{k=1}^{q}y_k.
	\end{equation}
\end{itemize}
Assuming that the regularization parameter for $\psi(\boldsymbol{x})=f_S(\boldsymbol{x})$ finding from $n_{fs}$ different values, and the parameter $\sigma$ for $\psi(\boldsymbol{x})=\sum_{k_2=1}^nK(\boldsymbol{x},\boldsymbol{x}_{S_k^i})$ has $n_{kernel}$ possible values. Then,
all above predicates totally could generate $(n_{fs}+n_{gs}+d+d(d+1)/2+n_{kernel}+1+d+1+d(d+1)/2+1)$ different statistical invariants due to the different training in each iteration, where $n_{fs}$, $n_{gs}$, and $n_{kernel}$ are the number of parameter in $\psi(\boldsymbol{x})=f_{\mathcal{S}^i}(\boldsymbol{x})$, $\psi(\boldsymbol{x})=g_{\mathcal{S}^i}(\boldsymbol{x})$, and $\psi(\boldsymbol{x})=\sum_{k_2=1}^nK(\boldsymbol{x},\boldsymbol{x}_{k_2(\mathcal{S}^i)})$. These statistical invariants consider the knowledge contained within the target domain itself, as well as the knowledge from the source domain, including sample, feature, and parameter.

\subsection{Time complexity}\label{Comparison of the Time Complexity of Related Transfer Learning Methods}
According to Algorithm \ref{alg:Al1}, the computational complexity of SETrLUSI is analyzed as follows:\par
1) The time complexity of line 6 is $\mathcal{O}(q)$. \par
2) The time complexity of line 7 is $\mathcal{O}(\gamma n_{\mathcal{S}^i})$.\par
3) The time complexity of line 8 is $\mathcal{O}(1+T_\psi)$, which is the time complexity of calculating one selected predicate. \par
4) The time complexity $\mathcal{O}(dp^2)$ in line 10 for a linear programming problem.\par
5) The time complexity of line 11 to line 13 is $\mathcal{O}(q+1+1)$, and line 18 is $\mathcal{O}(NlogN+1)$ for finding the minima weak learner.\par

Therefore, for $N$ source domains and $H$ iterations, the time complexity of SETrLUSI is $\mathcal{O}(HN(dq^2+T_\psi))$, where $\textup{max}\{q,~NlogN,~\gamma n_{\mathcal{S}^i}\}\ll dq^2$.
Now, we present a simple comparison between our proposed model and other related transfer learning models. \par

From the Table \ref{timecompare},  TrAdaboost faces a significant increase in time complexity when applied to multi-source domain as a single-source model. MSDTrAdaboost, and METL, being two-phase models, have their time complexity affected by two factors: the number of clusters and the involvement of multiple base classifiers. Moreover, the MSDTrAdaboost, MSTrAdaboost and WMSTrAdaboost have similar time complexities, as they primarily spend time on training the weak classifier SVM. SETrLUSI incorporates source domain samples into the training process as statistical invariants, requiring only a small amount of labeled target domain data, which results in a much lower time complexity compared to the above methods, but is influenced by dimension $d$ and the number of $\psi$. Using logistic regression as the base learner, the time complexity of 3SW\_MSTL is comparable with SETrLUSI.\par

\subsection{More details about experimental setting}\label{More details about experimental setting}
In this section, we give more details about experitmental setting for anyone who want to reproduce our results.\par

All boosting models use a linear SVM as the weak classifier, with the regularization parameter selected from \( \{ 2^{-8}, 2^{-6}, \dots, 2^8 \} \). For MHT, the number of clusters \( k \) in \( K \)-means is chosen from the range of 1 to 5, and METL uses the same three base classifiers as in \cite{yang2020multi}. For our SETrLUSI, the parameters \( \tau \) and \( \gamma \) are selected from \( \{ 0.3, 0.5, 0.7, 0.9 \} \) and \( \{ 0.1, 0.3, \dots, 0.9 \} \), respectively. For the part of ensemble learning, the number of iterations is set to 100 and we set the error to 0.499 when eqach weak learner is greater than 0.5, and to 0.001 when it is equal to 0 to avoid experiment interruption. Meanwhile, we take the re-sampling strategy to avoid the interruption of model training caused by the insufficient number of training samples. 

\subsection{Data description}\label{Data Description}

\textbf{UCI}:
For demonstrating the performance of 
SETrLUSI and other comparative models under different distributions between the source and target domains, the Rice \footnote{\url{https://archive.ics.uci.edu/dataset/545/rice+cammeo+and+osmancik}}, Wave \footnote{\url{https://archive.ics.uci.edu/dataset/107/waveform+database+generator+version+1}}, and Two Norm (TW) \footnote{\url{https://www.kaggle.com/datasets/timrie/twonorm}} three UCI datasets are used to construct transfer tasks. By using the $k$-means algorithm, we cluster each dataset into three domains based on three different features, as shown in Table \ref{UCICon} and Figure \ref{UCI}. \par
\setcounter{table}{4}
\begin{table}[h]
	\centering
	\begin{threeparttable}
		\centering
		\fontsize{8}{7}\selectfont
		\caption{The construction of multi-source UCI datasets}
		\label{UCICon}
		\begin{tabular}{c|ccccc}
			\toprule
			Dataset & Dimension & Cluster Feature & \#.$D_{\mathcal{S}^1}$ & \#.$D_{\mathcal{S}^2}$ & \#.$D_{\mathcal{T}}$ \\
			\midrule
			\multirow{3}{*}{Rice} & \multirow{3}{*}{7}& 1\&2\&3 & 1472  & 1134 & 1204 \\
			&& 3\&4\&7 & 1437  & 1149 & 1224 \\
%			&& 1\&3\&4 & 1300  & 1135 & 1375 \\
			\midrule
			\multirow{3}{*}{Wave} &\multirow{3}{*}{21}& 1\&7\&15 & 1818  & 1589 & 1593 \\
			&& 1\& 2\&4 & 1601  & 1541 & 1858 \\
%			&& 1\& 7\&19 & 1579  & 1529 & 1892 \\
			\midrule
			\multirow{3}{*}{TW} & \multirow{3}{*}{20}& 1\&3\&16 & 2887  & 2158 & 2355\\
%			&& 1\&15\&20 & 2269  & 2954 & 2177 \\
			&& 10\&13\&20 & 2432 & 2753 & 2215 \\
			\bottomrule
		\end{tabular}
	\end{threeparttable}
\end{table}

\begin{figure}[H]
	\centering	

%	\begin{minipage}{0.32\linewidth}
%		\centering
%		\includegraphics[width=1\linewidth]{rice134}
%		\caption*{\scriptsize(c) Rice 1\&3\& 4}
%	\end{minipage}
	
%	\begin{minipage}{0.32\linewidth}
%		\centering
%		\includegraphics[width=1\linewidth]{wave1715}
%		\caption*{\scriptsize(d) Wave 1\&7\&15}
%	\end{minipage}
	\begin{minipage}{0.32\linewidth}
		\centering
		\includegraphics[width=1\linewidth]{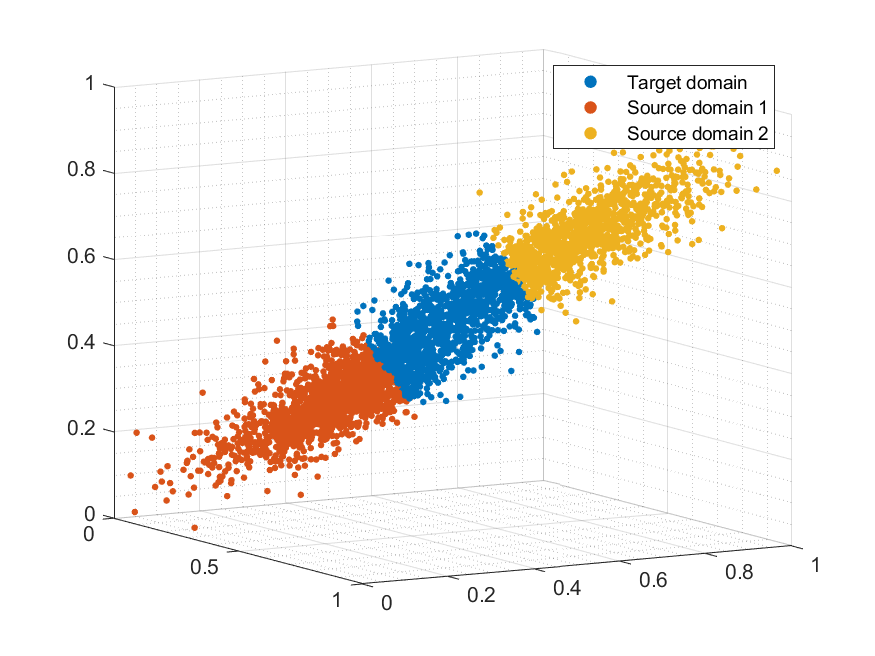}
		\caption*{\scriptsize(a) Rice 1\&2\&3}
	\end{minipage}
	\begin{minipage}{0.32\linewidth}
		\centering
		\includegraphics[width=1\linewidth]{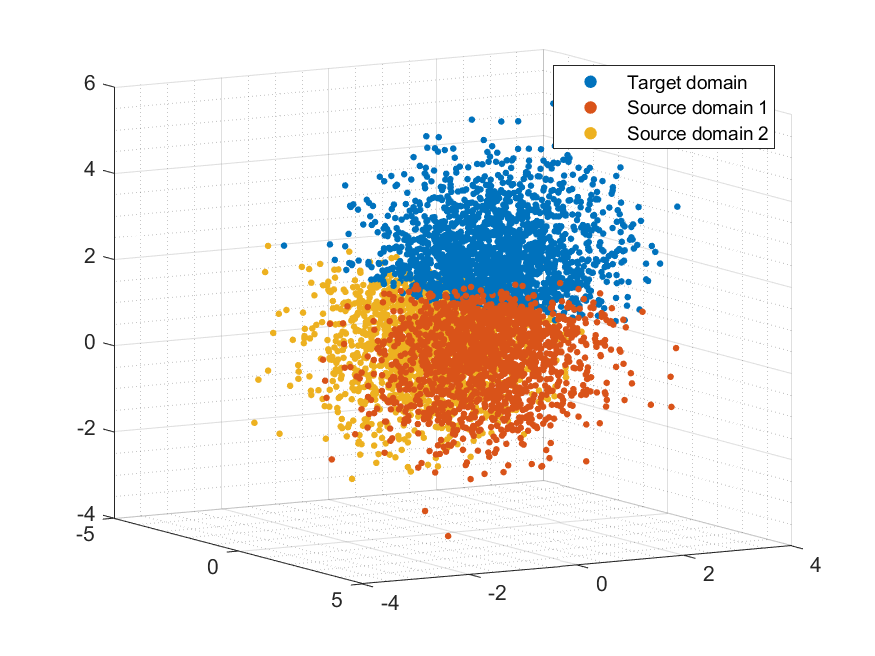}
		\caption*{\scriptsize(b) Wave 1\&2\&4}
	\end{minipage}
	\begin{minipage}{0.32\linewidth}
		\centering
		\includegraphics[width=1\linewidth]{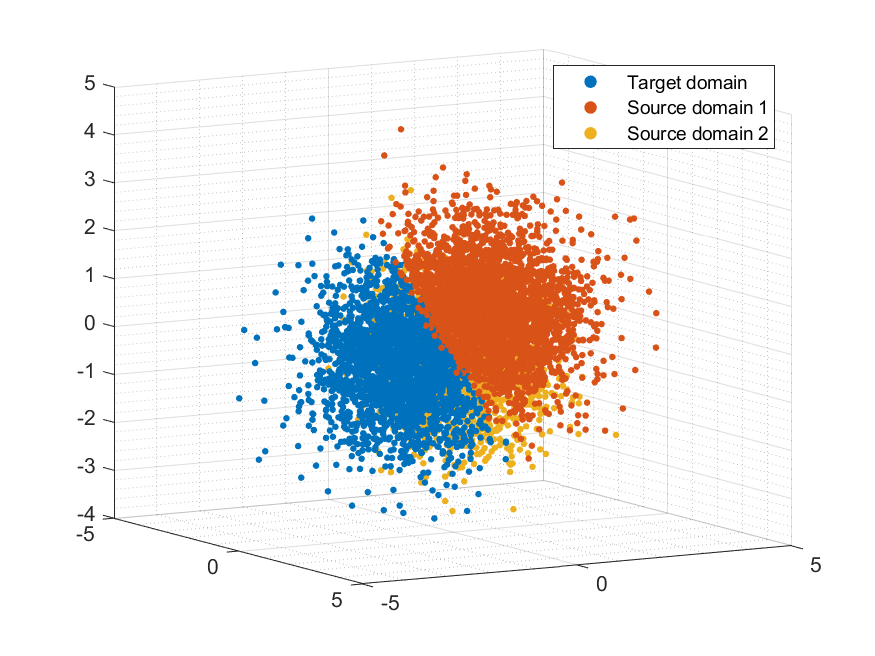}
		\caption*{\scriptsize(c) TW 1\&3\&16}
	\end{minipage}
	
	\begin{minipage}{0.32\linewidth}
		\centering
		\includegraphics[width=1\linewidth]{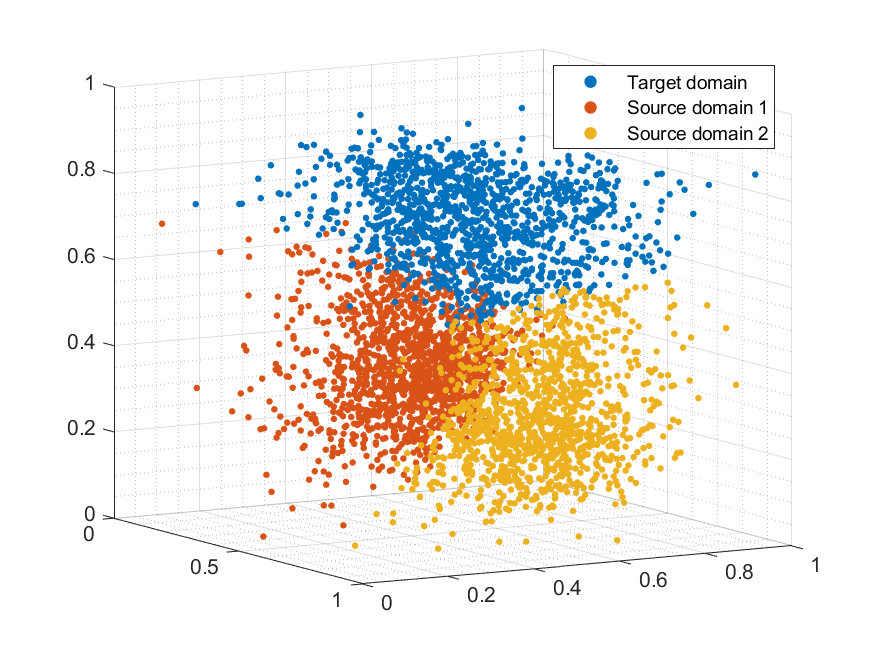}
		\caption*{\scriptsize(d) Rice 3\&4\&7}
	\end{minipage}
	\begin{minipage}{0.32\linewidth}
		\centering
		\includegraphics[width=1\linewidth]{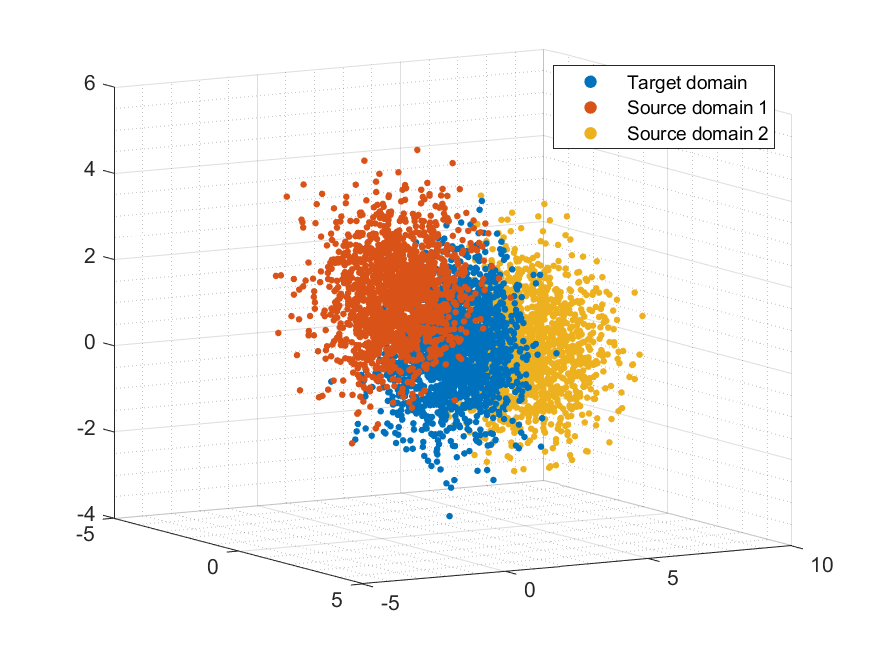}
		\caption*{\scriptsize(e) Wave 1\&7\&19}
	\end{minipage}
	\begin{minipage}{0.32\linewidth}
		\centering
		\includegraphics[width=1\linewidth]{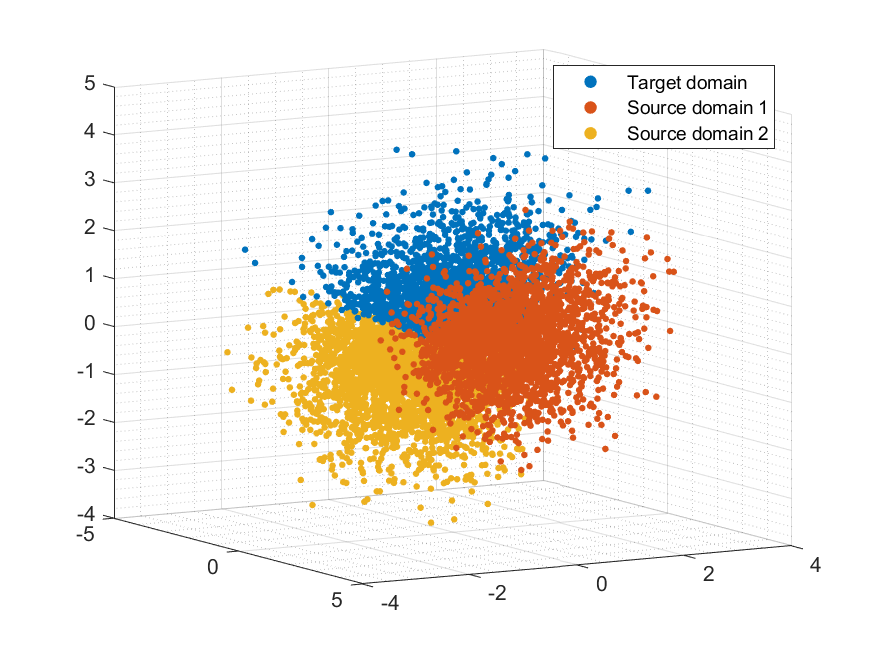}
		\caption*{\scriptsize(f) TW 10\&13\&20}
	\end{minipage}

%	\begin{minipage}{0.32\linewidth}
%		\centering
%		\includegraphics[width=1\linewidth]{TW11520}
%		\caption*{\scriptsize(h) Two Norm 1\&15\&20}
%	\end{minipage}

	\caption{These images respectively display 3-D plots of three UCI datasets clustered using $k$-means (k=3) on three features. The blue points represent the target domain, while the red and yellow points represent two source domains. Specifically, (a) and (d) are from the Rice, (b) and (e) are from the Wave, (c) and (f) are from the TwoNorm.}
	\label{UCI}
\end{figure}

\textbf{20 News}:
The 20 News \footnote{\url{http://qwone.com/~jason/20Newsgroups/}} consists of news articles covering 20
different topics and is published in the proceedings of the 12th International Conference on Machine Learning. It has since become
one of the international standard datasets used for research in
transfer learning. In this experiment, the construction of the
transfer tasks follows the same approach as \cite{duan2012domain}, where the four largest main categories ("comp", "rec", "sci", and "talk") are chosen for evaluation, and the dimension of each domain is 500.  Table \ref{20Ct} gives more details about eight multi-source tasks.\par
\begin{table}[h]
	\centering
	\begin{threeparttable}
		\caption{The construction of eight transfer tasks from 20 News}
		\label{20Ct}
		\fontsize{9}{8}\selectfont
		\begin{tabular}{c|ccccc}
			\toprule
			Task & Dimension &Source Domain 1 & Source Domain 2  & Source Domain 3 & Target domain \\
			\midrule
			CS2R & \multirow{2}{*}{500} &c.w.x \& r.s.h  & c.s.i.p.h \& s.m & \textbackslash & c.s.i.p.h \& r.m \\
			CST2R & & c.w.x \& r.s.h  & c.s.i.p.h \& s.m & c.w.x \& t.p.m & c.s.i.p.h \& r.m \\
			\midrule
			CR2S & \multirow{2}{*}{500}&c.w.x \& s.c & c.s.i.p.h \& r.m & \textbackslash & c.s.i.p.h \& s.m \\
			CRT2S & &c.w.x \& s.c & c.s.i.p.h \& r.m & c.w.x \& t.p.m & c.s.i.p.h \& s.m \\
			\midrule
			RS2C & \multirow{2}{*}{500}&c.s.i.p.h \& r.m & c.w.x \& s.m & \textbackslash & c.w.x \& r.s.h  \\
			RST2C & & c.s.i.p.h \& r.m & c.w.x \& s.m & c.w.x \& t.p.m & c.w.x \& r.s.h  \\
			\midrule
			CS2T & \multirow{2}{*}{500}&c.w.x \& t.p.m & c.w.x \& s.m & \textbackslash & c.s.i.p.h \& t.p.g \\
			CSR2T & &c.w.x \& t.p.m & c.w.x \& s.m & c.w.x \& r.s.h & c.s.i.p.h \& t.p.g \\
			\bottomrule
		\end{tabular}
	\end{threeparttable}
\end{table}

\textbf{VLSC object Image classification}:
VLSC \footnote{\url{https://github.com/HDSYW/transferlearning/blob/master/data/dataset.md\#vlsc}} contains four domains: V (VOC2007) \cite{everingham2010pascal}, L (LabelMe) \cite{russell2008labelme}, S (SUN09) \cite{choi2010exploiting} and C (Caltech) \cite{griffin2007caltech}, which are well-known object image recognition datasets. There are 5 classes: "bird", "car", "chair", "dog", and "person" in these datasets, and each of them has 4096 dimensions. For our binary classification problem, we choose "car" as class 1 and "chair" as class 0, shown in Figure \ref{VLSCFig}, where more details can be found in Table \ref{VLSCCon}.\par
\begin{table}[h]
	\centering
	\begin{threeparttable}
		\caption{The construction of transfer task on VLSC image classification}\			
		\label{VLSCCon}
		\fontsize{7}{7}\selectfont
		\begin{tabular}{c|ccccc}
			\toprule
			Task & Dimension &$D_{\mathcal{S}^1}$ & $D_{\mathcal{S}^2}$  & $D_{\mathcal{S}^3}$ & $D_{\mathcal{T}}$ \\
			\midrule
			SV2C & \multirow{2}{*}{4096} &SUN09   & VOC2007 & \textbackslash & Caltech \\
			LSV2C & &LabelMe  & SUN09  & VOC2007 & Caltech \\
			\midrule
			LV2S & \multirow{2}{*}{4096}&LabelMe & VOC2007 & \textbackslash & SUN09 \\
			LCV2S & &LabelMe & Caltech & VOC2007 & SUN09  \\
			\midrule
			SC2V & \multirow{2}{*}{4096}&SUN09 & Caltech & \textbackslash & VOC2007 \\
			SCL2V & &SUN09 & Caltech & LabelMe & VOC2007  \\
			\bottomrule
		\end{tabular}
	\end{threeparttable}
\end{table}

\begin{figure}[H]
	\centering	
	\begin{minipage}{0.40\linewidth}
		\centering
		\includegraphics[width=0.40\linewidth]{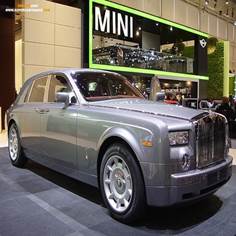}
		\includegraphics[width=0.40\linewidth]{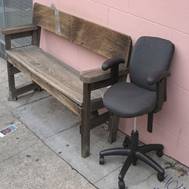}
		\caption*{(a) VOC2007}
	\end{minipage}
	\begin{minipage}{0.40\linewidth}
		\centering
		\includegraphics[width=0.40\linewidth]{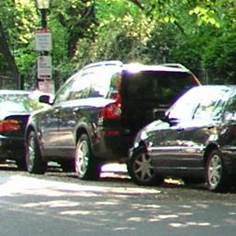}
		\includegraphics[width=0.40\linewidth]{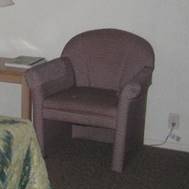}
		\caption*{(b) LabelMe}
	\end{minipage}
	
	\begin{minipage}{0.40\linewidth}
		\centering
		\includegraphics[width=0.40\linewidth]{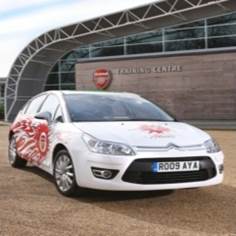}
		\includegraphics[width=0.40\linewidth]{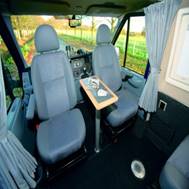}
		\caption*{(c) SUN09}
	\end{minipage}
	\begin{minipage}{0.40\linewidth}
		\centering
		\includegraphics[width=0.40\linewidth]{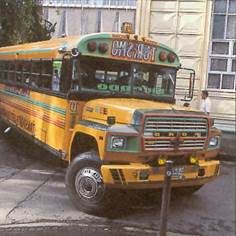}
		\includegraphics[width=0.40\linewidth]{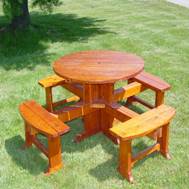}	
		\caption*{(d) Caltech}
	\end{minipage}
	\caption{These images respectively give the expamles of the Class 1 "car" and Class 0 "chair" in four domains.}
	\label{VLSCFig}
\end{figure}

\end{document}